\documentclass{article}
\usepackage{abc2027_conference,times}
\usepackage{amsmath,amssymb,amsthm,bm}
\usepackage{graphicx,booktabs,algorithm,algorithmic}
\usepackage{hyperref}
\usepackage{url,comment}
\newtheorem{theorem}{Theorem}[section]
\newtheorem{lemma}[theorem]{Lemma}
\newtheorem{proposition}[theorem]{Proposition}
\newtheorem{corollary}[theorem]{Corollary}
\theoremstyle{remark}

\theoremstyle{plain}
\input{Definitions}
\newcommand{\xstar}{\xb^{\star}}
\newcommand{\xhatstar}{\widehat{\xb}^{\star}}
\newcommand{\dlam}{d_{\lambda}}
\newcommand{\Slam}{\Sigmab_{\lambda}}
\newcommand{\smin}{\sigma_{\min}}
\newcommand{\smax}{\sigma_{\max}}

\renewcommand{\Ehat}{\widehat{\Eb}}
\renewcommand{\Rhat}{\widehat{\Rb}}
\renewcommand{\II}{\mathbf{1}}
\hypersetup{hidelinks,pdftitle={Fresh Sampling for Ridge Regression: Convergence and Adaptive Probabilities},pdfauthor={}}
\title{The Advantages of Fresh Sketching for Ridge Regression}

\author{Linkai Ma \\
Purdue University\\
\texttt{ma856@purdue.edu} \\
\And
Qilin Li \\
University of Wisconsin-Madison \\
\texttt{qli478@wisc.edu} \\
\And
Petros Drineas \\
Purdue University \\
\texttt{pdrineas@purdue.edu}\\
}
\abcfinalcopy
\begin{document}
\maketitle
\lhead{}
\renewcommand{\headrulewidth}{0pt}
\begin{abstract}

Over the past 25 years, sketching and sampling have become widely used tools for accelerating large-scale regression. In iterative randomized solvers, a basic design choice is whether to $\textit{reuse}$ the same sketch or draw $\textit{fresh}$ randomness at every step. For (under-constrained) iterative ridge regression with column sampling, whether fresh sketches offer provable advantages has remained open: $\textit{We show that they do.}$ Fresh sketching lets us analyze error only along the current residual solution, rather than uniformly over the entire Gram matrix. This directional view yields sharper convergence guarantees for leverage score and ridge leverage score sampling and, more importantly, leads to residual-aware sampling rules. By minimizing the variance of the relevant sketched matrix-vector product, we derive an oracle distribution and practical approximations to the oracle distribution, including a mixture sampling distribution with (somewhat weaker) convergence guarantees. Experiments on synthetic and real data, including ridge probes on Qwen2.5 representations, support our theory, showing substantially faster convergence.

\end{abstract}
\section{Introduction}\label{sec:intro}

Tools from Randomized Numerical Linear Algebra (RandNLA), such as sketching and sampling, are of paramount importance for accelerating large-scale regression and other matrix computations~\citep{halko2011finding,woodruff2014sketching,drineas2017lectures,martinsson2020randomized,derezinski2021determinantal, murray2023randomized, mahoney2011randomized,drineas2016randnla}. By replacing expensive matrix operations with smaller randomized approximations, these methods can substantially reduce computational and memory costs while retaining provable accuracy guarantees.

A basic design choice in an iterative randomized solver is whether to draw a sketch once and reuse the same randomness throughout the iteration, or to draw fresh randomness at every step. While reusing a sketch is natural computationally, fresh sketches may better capture the changing directions encountered
by the iterative method. This raises a basic question:

\begin{quote}
\emph{When, and why, does fresh randomness help an iterative sketched
solver?}
\end{quote}

We study this question for iterative ridge regression with column
sampling. Given a design matrix $\Ab\in\mathbb{R}^{n\times d}$, a
response vector $\bb\in\mathbb{R}^{n}$, and a regularization parameter
$\lambda>0$, ridge regression solves
\begin{equation}
 \xstar=\argmin_{\xb\in\R^d}
 \big\{\twonorm{\Ab\xb-\bb}^{2}+\lambda\twonorm{\xb}^{2}\big\}
 =\Ab^{\ts}(\Ab\Ab^{\ts}+\lambda\Ib_n)^{-1}\bb.
 \label{eq:ridge}
\end{equation}
\citep{hoerl1970ridge,saunders1998ridge}.
We focus on the wide regime $n\ll d$, where the number of features is much larger than the number of observations. Directly forming the Gram matrix $\Ab \Ab^\top$ costs $O(n^2d)$, which can become prohibitive when $d$ is very large.

Column sampling reduces this cost by sampling and rescaling features
to approximate the Gram matrix using a much smaller number of columns
\citep{drineas2006fast}. \citet{chowdhury2018iterative}
use this approximation in an iterative solver that repeatedly solves a sketched ridge regression problem and corrects the remaining residual. For sufficiently accurate sketches, these residual corrections yield geometric decay of the solution error.

The analysis of \citet{chowdhury2018iterative} uses a fixed initial sketch that is reused across iterations. Their experiments, however, indicate that drawing a new sample at every iteration can accelerate convergence and reduce the sketch size needed to converge, and they explicitly ask whether this additional randomness leads to stronger theoretical guarantees. \citet{kacham2022sketching} subsequently analyze the same iterative solver with refreshed oblivious sketches and obtain improved sample complexity guarantees. The proof of their Lemma 3.3 analyzes the same matrix-vector product as our Lemma~\ref{lem:colev-vector-var}. More recently, \citet{li2024optimal} propose a sampling probability similar to our oracle proposal: their ideal rule is $p_i^{\mathrm{LN},(j)} \propto |\xb_i^{\star(j)}|\twonorm{\Ab_{\ast i}}$, while ours is $p_i^{\star(j)} \propto |\xb_i^{\star(j)}|\sqrt{w_i}$. We return to the column sampling setting and study fresh leverage score and ridge leverage score sketches, which exploit the structure of the design matrix and are natural sampling distributions for ridge regression. For a more detailed discussion of related works, we refer the readers to Section~\ref{sec:related}.

Our main observation is that \emph{fresh randomness changes what must
be approximated}. Existing fixed-sketch analyses seek sufficiently
accurate uniform approximation of the relevant matrix, for example by
controlling the norm of the matrix $\Vb^\top \Sb \Sb^\top \Vb-\Ib$,
where $\Vb \in \mathbb{R}^{d \times n}$ is the tall-and-thin matrix of the right singular vectors of $\Ab$ and the $d \times s$ matrix $\Sb$ (with $s\ll d$) is the sketching matrix (see Section~\ref{app:algorithms} and~\ref{app:fixed-comparison} for details).
With an independently refreshed sketch at each iteration, this is
stronger than necessary. At iteration $j$, the solver only needs the
sketch to be accurate in the direction of the current residual ridge
solution. The key quantity is therefore a sketched
\emph{matrix-vector} product rather than a uniform
matrix-matrix approximation. Conditional independence of the fresh
sketch from the previous iterations allows us to analyze the variance
of this directional error and propagate it through the iteration.

This directional viewpoint has a second consequence: it tells us how
the sampling probabilities themselves should change during the
iteration. Once convergence is governed by the variance of a particular
sketched matrix-vector product, we can choose the sampling distribution
to minimize that variance. This produces an oracle distribution that
combines ridge leverage information with the current residual solution.
The oracle solution is unavailable in practice, but its structure
suggests adaptive sampling probabilities based on an inexpensive
approximation. Thus, residual-aware sampling is not a separate heuristic:
it emerges directly from the same directional analysis that explains
the advantages of fresh sketching.

\paragraph{Our contributions.}
We establish both analytical and algorithmic advantages of fresh
sketching for iterative ridge regression with column sampling.

\begin{enumerate}

\item
\textbf{Fresh sketches yield sharper convergence through directional error.}
We derive solution error bounds for fresh leverage score and ridge leverage score sampling in Theorems~\ref{thm:colev-ms} and~\ref{thm:ridgelev}. When $\|\Ab\|_2^2/\lambda$ is moderate, these results achieve geometric
convergence at rate $\varepsilon^{t}$ (where $\varepsilon$ is the desired relative error and $t$ is the number of iterations) using $O(n/\varepsilon^2)$ and $O(d_\lambda/\varepsilon^2)$ sampled columns per iteration, respectively, where $d_\lambda$ is the effective degrees of freedom. Compared with \citet{chowdhury2018iterative}, these bounds save factors of $\log n$ and $\log(1+d_\lambda)$ while retaining the same geometric rate. Secondly, at the same sample sizes, Theorems~\ref{thm:colev-struct} and~\ref{thm:ridgelev-struct} give faster guaranteed convergence: the contraction factor gains an additional inverse square root of the corresponding logarithm. Finally, for leverage score sampling, Theorem~\ref{thm:colev} shows that independence and cancellation
 of the linear sketch error yield the algorithmic bias bound
 $\|\ex[\xhatstar-\xstar]\|_2\le\varepsilon^{2t}\|\xstar\|_2$
 at the stated sample size.

\noindent\textit{Remark.} We note that, in general, independent column sampling methods with replacement cannot avoid logarithmic factors; this has been known since the early days of random sampling and follows from coupon collector considerations~\citep[see][Remark~42]{drineas2017lectures}. Recently, celebrated results use modern random matrix theory (RMT) tools to remove the logarithmic factor from the embedding dimension for sparse oblivious sketches, resolving the conjecture of \citet{nelson2013osnap} up to sub-polylogarithmic factors~\citep{chenakkod2026optimal,DBLP:conf/icalp/ChenakkodDD25,chenakkod2024optimal,brailovskaya2024universality}.

Our result, to the best of our knowledge, is the first one to avoid logarithmic factors in the sampling size complexity for the iterative ridge regression framework of~\citet{chowdhury2018iterative}, by using fresh randomness. Our proofs leverage conditional probability and standard measure concentration inequalities.

\item
\textbf{Directional error leads to residual-aware sampling.}
We minimize the variance of the matrix-vector product that governs the
one-step error and derive an oracle sampling distribution that combines
ridge leverage scores with the current residual solution. We additionally introduce practical approximations of the oracle sampling distribution. Our proposed adaptive method uses an initial ridge pilot sketch to approximate the residual solution; mixture methods further mix this probability distribution with ridge leverage scores to improve practicality and performance. (See Section~\ref{sec:oracle} for details.)

\item
\textbf{Experiments demonstrate the advantages of fresh and
residual-aware sketches.}
We compare uniform, leverage score, ridge leverage score, oracle,
adaptive, and mixture sampling on ARCENE and synthetic data
(Figures~1 and~3), and separately compare fixed and fresh sketches on
ARCENE (Figure~4), extending the experiments of \citet{chowdhury2018iterative}. We also evaluate the method on semantic similarity ridge probing for the STS Benchmark, using concatenated representations from multiple layers of Qwen2.5 to produce a wide regression problem~\citep{cer2017stsb,qwen2025technical}. Across iteration counts, sketching budgets, and regularization parameters, the experiments show the convergence advantages of the residual-aware sampling rules.

\end{enumerate}

\section{Problem setup and iterative solver}
\label{sec:setup}
Scalars are denoted by lowercase letters (e.g., $x$, $\lambda$); vectors by bold lowercase letters (e.g., $\xb$, $\gb$); matrices by bold uppercase letters (e.g., $\Wb$, $\Gb$). For a matrix $\Ab \in \RR{m}{n}$ we write $\|\Ab\|_F$ for the Frobenius norm, $\twonorm{\Ab}$ for the spectral norm. The notation $\Ab_{i\ast}$ and $\Ab_{\ast i}$ denotes the $i$th row and column, respectively. The order $\succeq$ denotes positive semidefinite ordering, and $\II_B$ denotes the indicator of an event $B$. Without loss of generality, throughout the analysis we assume that all sampling probabilities are strictly positive. (We will always skip outcomes with zero sampling probabilities.)
Let $\Ab\in\RR{n}{d}$ have full row rank, with $d>n>1$, and let
$\bb\in\R^n$ and $\lambda>0$. Write $\Ib_m$ for the $m\times m$ identity
and $\|\cdot\|_2$ for the Euclidean vector norm. 

We consider the ridge regression problem in~\eqref{eq:ridge}, write the thin singular value decomposition as
$\Ab=\Ub\Sigmab\Vb^{\ts}$, where $\Ub\in\RR{n}{n}$ is orthogonal,
$\Vb\in\RR{d}{n}$ has orthonormal columns, and
$\Sigmab=\diag(\sigma_1,\ldots,\sigma_n)$ with
$\smax=\sigma_1\ge\cdots\ge\sigma_n=\smin>0$.
 Define $ \Slam=\Sigmab(\Sigmab^2+\lambda\Ib_n)^{-1/2},
 \dlam=\|\Slam\|_F^2=\sum_{i=1}^n\frac{\sigma_i^2}{\sigma_i^2+\lambda}.$ Thus $\twonorm{\Slam}\le1$ and $0<\dlam<n$. The quantity $\dlam$
is the effective degrees of freedom \citep{alaoui2015fast}.

\paragraph{Column sampling.}
A sketch of size $s$ is a matrix $\Sb\in\RR{d}{s}$ that independently samples
indices $i_1,\ldots,i_s$ from  $\{1, 2, \ldots, d\}$ using a probability distribution $p=(p_1,\ldots,p_d)$ and has columns
$\Sb_{\ast k}=\eb_{i_k}/\sqrt{s p_{i_k}}$
\citep{drineas2006fast}.
The two baseline distributions use leverage scores (LS)
\citep{drineas2012leverage} and
ridge leverage scores (RLS) \citep{alaoui2015fast,cohen2017ridge}:
\begin{align*}
 p_i^{\mathrm{ls}}=\frac{\twonorm{\Vb_{i\ast}}^2}{n}, \quad p_i^{\mathrm{rls}}=\frac{w_i}{\dlam},
 w_i:=\twonorm{(\Vb\Slam)_{i\ast}}^2.
\end{align*}

\paragraph{Residual correction.}
We use the iteration of \citet{chowdhury2018iterative}, with a new sketch
at every step. Fix an iteration count $t$ and initialize $\bb^{(0)}=\bb$,
$\widetilde{\xb}^{(0)}=\zero_d$, and $\yb^{(0)}=\zero_n$.
At iteration $j=1,\ldots,t$, update the residual, draw a fresh sketch
$\Sb^{(j)}$ with size $s^{(j)}$ and probabilities $p^{(j)}$, and compute
\begin{equation}
\begin{split}
 \bb^{(j)}&=\bb^{(j-1)}-\lambda\yb^{(j-1)}-\Ab\widetilde{\xb}^{(j-1)},\\
 \Hb_j&=\Ab\Sb^{(j)}\Sb^{(j)\ts}\Ab^{\ts}+\lambda\Ib_n,\qquad
 \yb^{(j)}=\Hb_j^{-1}\bb^{(j)},\\
 \widetilde{\xb}^{(j)}&=\Ab^{\ts}\yb^{(j)}.
\end{split}
\label{eq:iteration}
\end{equation}

After $t$ steps, the output is $\xhatstar=\sum_{j=1}^{t}\widetilde{\xb}^{(j)}$.
The exact solution of residual problem $j$ is
$\xb^{\star(j)}=\Ab^{\ts}(\Ab\Ab^{\ts}+\lambda\Ib_n)^{-1}\bb^{(j)}$.
In particular, $\bb^{(1)}=\bb^{(0)}=\bb$ and
$\xb^{\star(1)}=\xb^{\star(0)}=\xstar$.
The history $\Fcal_j=\sigma(\Sb^{(1)},\ldots,\Sb^{(j)})$ is the sigma algebra
generated by the first $j$ sketches, with
$\Fcal_0=\{\emptyset,\Omega\}$, where $\Omega$ is the sample space.
Thus $\xb^{\star(j)}$ is
$\Fcal_{j-1}$-measurable.
For the baseline distributions, sketches are independent across iterations.
For residual-aware distributions, each sketch depends on past randomness.
All expectations refer to sketch randomness, with $\Ab$ and $\bb$ fixed.

The identities
$\xb^{\star(j)}=\xb^{\star(j-1)}-\widetilde{\xb}^{(j-1)}$ and
$\xhatstar-\xstar=\widetilde{\xb}^{(t)}-\xb^{\star(t)}$
reduce the final error to successive residual errors
(Lemma~\ref{lem:chow-rec}, proved in
\hyperref[proof:residual]{Appendix~\ref*{app:residual-identities}}).
Algorithm~\ref{alg:fresh} states the fresh sketching iterations, including leverage score, ridge leverage score and orcale sampling. Adaptive and mixture sampling algorithms appear in Appendix~\ref{app:algorithms}. For completeness, the exact condition and convergence results of \citet{chowdhury2018iterative} are summarized in Appendix~\ref{app:fixed-comparison}.

\begin{algorithm}[t]
\caption{Iterative ridge regression with fresh sketches}
\label{alg:fresh}
\begin{algorithmic}[1]
\STATE \textbf{Input:} $\Ab,\bb$; $\lambda>0$; iteration count $t\ge1$; sketch sizes $s^{(1)},\ldots,s^{(t)}$; rule LS, RLS, or Oracle.
\STATE Initialize $\bb^{(0)}=\bb$, $\widetilde{\xb}^{(0)}=\zero_d$, and $\yb^{(0)}=\zero_n$.
\FOR{$j=1,\ldots,t$}
\STATE Set $\bb^{(j)}=\bb^{(j-1)}-\lambda\yb^{(j-1)}-\Ab\widetilde{\xb}^{(j-1)}$.
\IF{rule is LS}
\STATE Set $p_i^{(j)}=\twonorm{\Vb_{i\ast}}^2/n$.
\ENDIF
\IF{rule is RLS}
\STATE Set $p_i^{(j)}=w_i/\dlam$.
\ENDIF
\IF{rule is Oracle}
\STATE Compute $\xb^{\star(j)}=\Ab^{\ts}(\Ab\Ab^{\ts}+\lambda\Ib_n)^{-1}\bb^{(j)}$.
\STATE Set $p_i^{(j)}\propto |\xb_i^{\star(j)}|\sqrt{w_i}$ .
\ENDIF
\STATE Draw a fresh $\Sb^{(j)}$ of size $s^{(j)}$ using Algorithm~\ref{alg:sample}.
\STATE Form $\Hb_j=\Ab\Sb^{(j)}\Sb^{(j)\ts}\Ab^{\ts}+\lambda\Ib_n$ and solve
$\Hb_j\yb^{(j)}=\bb^{(j)}$.
\STATE Set $\widetilde{\xb}^{(j)}=\Ab^{\ts}\yb^{(j)}$.
\ENDFOR
\STATE \textbf{Return:} $\xhatstar=\sum_{j=1}^{t}\widetilde{\xb}^{(j)}$.
\end{algorithmic}
\end{algorithm}
\section{Leverage Sampling}
\label{sec:colev}

We first consider fresh leverage score sketches, with
$p_i=\twonorm{\Vb_{i\ast}}^2/n$ and deterministic sketch sizes
$s^{(1)},\ldots,s^{(t)}$. Define
\begin{align}
  \Ehat^{(j)}&=\Vb^{\ts}\Sb^{(j)}\Sb^{(j)\ts}\Vb-\Ib_n,
  \quad\Gb=\Ib_n+\lambda\Sigmab^{-2}, \notag\\
  \Rhat^{(j)}&=(\Ib_n+\Gb^{-1}\Ehat^{(j)})^{-1}-\Ib_n
  =-(\Gb+\Ehat^{(j)})^{-1}\Ehat^{(j)}. \label{eq:R}
\end{align}
The inverse exists for every sketch because
$\Gb+\Ehat^{(j)}\succeq\lambda\Sigmab^{-2}\succ\zero$.
The next two lemmas give the error identity and the variance along one vector.

\begin{lemma}
\label{lem:telescope}
For any sequence of sketches in \eqref{eq:iteration},
\begin{equation}
  \xhatstar-\xstar
  =(-1)^{t-1}\Vb\Rhat^{(t)}\cdots\Rhat^{(1)}\Vb^{\ts}\xstar.
  \label{eq:telescope}
\end{equation}
\end{lemma}

\begin{lemma}
\label{lem:colev-vector-var}
For independent leverage score sketches and every
$\Fcal_{j-1}$-measurable vector $\zb\in\R^n$ with finite second
moment,
\begin{equation}
  \ex\!\left[\twonorm{\Ehat^{(j)}\zb}^{2}\mid\Fcal_{j-1}\right]
  =\frac{n-1}{s^{(j)}}\twonorm{\zb}^{2}.
  \label{eq:col-vector-var-cond}
\end{equation}
\end{lemma}

Appendix~\ref{sec:telescope}
contains the \hyperref[proof:column-product]{proof of (\ref*{eq:telescope})}
and the \hyperref[proof:column-variance]{proof of (\ref*{eq:col-vector-var-cond})}. The above product identifies the key matrix-vector product that quantifies the final solution quality of the iterative ridge regression solver. Instead of trying to obtain uniform control over $\twonorm{\Vb^\top \Sb \Sb^\top \Vb -\Ib_n}$ as in~\citet{chowdhury2018iterative}, we only need to control $\twonorm{\Ehat^{(j)} \zb} = \twonorm{\Vb^\top \Sb \Sb^\top \Vb \zb-\zb}$. We now present our improved convergence result for leverage score sampling.

\begin{theorem}
\label{thm:colev-ms}
Running Algorithm~\ref{alg:fresh} with LS gives
\begin{equation}
  \ex\!\left[\twonorm{\xhatstar-\xstar}^{2}\right]
  \le
  \left(\frac{\smax^2}{\lambda}\right)^{2t}
  \prod_{j=1}^{t}\frac{n-1}{s^{(j)}}\twonorm{\xstar}^{2}.
  \label{eq:colev-ms}
\end{equation}
For any fixed $t$ and $\delta_0\in(0,1)$, with probability at least
$1-\delta_0$,
\begin{equation}
  \twonorm{\xhatstar-\xstar}
  \le \delta_0^{-1/2}
  \left(\frac{\smax^2}{\lambda}\right)^t
  \prod_{j=1}^{t}\sqrt{\frac{n-1}{s^{(j)}}}\twonorm{\xstar}.
  \label{eq:colev-hp}
\end{equation}
In particular, for $\varepsilon\in(0,1)$, a common sketch size
\begin{equation}
  s\ge\left(\frac{\smax^2}{\lambda}\right)^2
  \frac{n-1}{\varepsilon^2}
  \label{eq:colev-var-scale}
\end{equation}
gives $\ex[\twonorm{\xhatstar-\xstar}^2]
\le\varepsilon^{2t}\twonorm{\xstar}^2$ and
$\twonorm{\xhatstar-\xstar}\le \delta_0^{-1/2} \varepsilon^t\twonorm{\xstar}$ with probability at least
$1-\delta_0$.
\end{theorem}

The proof in \hyperref[proof:column-ms]{Appendix~\ref*{app:column-ms}} uses
$\twonorm{(\Gb+\Ehat^{(j)})^{-1}}\le\smax^2/\lambda$ in
\eqref{eq:R}, then applies \eqref{eq:col-vector-var-cond} conditionally at
each iteration. One application of Markov's inequality to the final error
gives \eqref{eq:colev-hp}; substituting \eqref{eq:colev-var-scale} gives
the common-size bounds.
The required size is
$O((\smax^2/\lambda)^2n/\varepsilon^2)$, where \citet{chowdhury2018iterative} requires $O(n\log n/\varepsilon^2)$. When $\smax^2/\lambda$ is controlled, fresh sketches can save a factor of $\log n$ for the required sketch size.

A separate bound requires larger sketch sizes and replaces $\smax^2/\lambda$ by
$(1-\varepsilon)^{-1}$. It yields an
additional contraction factor of size $O(\frac{1}{\sqrt{\log n}})$
(Theorem~\ref{thm:colev-struct}, proved in \hyperref[proof:column-embedding]{Appendix~\ref*{app:column-embedding}}).

\textbf{Bias and sharper contraction bounds.}
We additionally analyze the algorithmic bias, i.e., the expected error vector
$\ex[\xhatstar-\xstar]$. Fix $0<\varepsilon<1$ and let
$\Mb_j=\Gb^{-1}\Ehat^{(j)}$,
$C=\frac{\smax^2}{\lambda}+\frac{\smax^2}{\smin^2}$, and
$\delta(s)=4(n+1)\exp(-3\varepsilon^2s/(32n))$.
For $s^{(j)}\ge8n/\varepsilon^2$, unbiasedness and an inverse
perturbation bound give
\begin{equation*}
 \ex[\Rhat^{(j)}]=\ex[(\Ib_n+\Mb_j)^{-1}\Mb_j^2],\qquad
 \twonorm{\ex[\Rhat^{(j)}]}
 \le\frac{\varepsilon^2}{2}+C(n-1)^2\delta(s^{(j)}).
\end{equation*}
Independence multiplies these factors in the algorithmic bias bound.
When $C(n-1)^2\delta(s^{(j)})\le\varepsilon^2/2$, this gives
$\twonorm{\ex[\xhatstar-\xstar]}\le\varepsilon^{2t}\twonorm{\xstar}$;
Theorem~\ref{thm:colev} states a sufficient sketch size. Its proof is deferred to Appendix~\ref{app:column-bias}.

\section{Ridge Leverage Sampling}
\label{sec:ridgelev}

We now consider independent ridge leverage score sketches, with $p_i=w_i/\dlam$ and sketch sizes $s^{(1)},\ldots,s^{(t)}$. Set $\Xb=\Vb\Slam$, $w_i=\twonorm{\Xb_{i\ast}}^2$, and
\begin{equation}
  \Eb^{(j)}=\Xb^{\ts}(\Sb^{(j)}\Sb^{(j)\ts}-\Ib_d)\Xb,
  \qquad c_\lambda=\frac{\smax^2+\lambda}{\lambda}.
  \label{eq:clambda}
\end{equation}
The next two lemmas relate the error in one residual solve to a sketched
matrix-vector product and bound its expected squared 2-norm.

\begin{lemma}
\label{lem:psdfloor}\label{lem:periter}
For every sketch, $\Ib_n+\Eb^{(j)}\succ0$ and
$\twonorm{(\Ib_n+\Eb^{(j)})^{-1}}\le c_\lambda$.
Moreover,
\begin{equation*}
 \xb^{\star(j)}-\widetilde{\xb}^{(j)}
 =\Xb(\Ib_n+\Eb^{(j)})^{-1}\Xb^{\ts}
      (\Sb^{(j)}\Sb^{(j)\ts}-\Ib_d)\xb^{\star(j)}.
\end{equation*}
Consequently, the norm of this error is at most
$c_\lambda\twonorm{\Xb^{\ts}(\Sb^{(j)}\Sb^{(j)\ts}-\Ib_d)
\xb^{\star(j)}}$. If $\twonorm{\Eb^{(j)}}\le\varepsilon<1$, we have the sharper
bound
\begin{equation}
 \twonorm{\xb^{\star(j)}-\widetilde{\xb}^{(j)}}
 \le\frac{1}{1-\varepsilon}
 \twonorm{\Xb^{\ts}(\Sb^{(j)}\Sb^{(j)\ts}-\Ib_d)\xb^{\star(j)}}.
 \label{eq:periter-struct}
\end{equation}
\end{lemma}

\begin{lemma}
\label{lem:ridge-vector-var}
For independent ridge leverage score sketches and every
$\Fcal_{j-1}$-measurable vector $\xb\in\R^d$ with
finite second moment,
\begin{equation}
    \ex\!\left[ \twonorm{\Xb^{\ts}(\Sb^{(j)}\Sb^{(j)\ts}-\Ib_d)\xb}^2
 \mid\Fcal_{j-1}\right]\le\frac{\dlam}{s^{(j)}}\twonorm{\xb}^2.\label{eq:ridge-main-variance}
\end{equation}

\end{lemma}

Appendix~\ref{sec:ridge-tools} contains the
\hyperref[proof:ridge-identity]{proof of Lemma~\ref*{lem:periter}} and the
\hyperref[proof:ridge-vector-var]{proof of Lemma~\ref*{lem:ridge-vector-var}}.
The first lemma isolates the sketch error applied to the current residual solution. The second controls that product conditionally on previous sketches. Together they give the following convergence result.

\begin{theorem}[Fresh ridge leverage sampling]
\label{thm:ridgelev}\label{cor:ridgelev-ms}\label{cor:ridgelev-var-scale}
Let $c_\lambda=(\smax^2+\lambda)/\lambda$ as in~\eqref{eq:clambda}.
Running Algorithm~\ref{alg:fresh} with RLS gives
\begin{equation}
  \ex\twonorm{\xhatstar-\xstar}^2
  \le c_\lambda^{2t}\prod_{j=1}^{t}\frac{\dlam}{s^{(j)}}
       \twonorm{\xstar}^2.
  \label{eq:ridge-ms-det}
\end{equation}
For every $\delta\in(0,1)$, with probability at least $1-\delta$,
\begin{equation}
  \twonorm{\xhatstar-\xstar}
  \le \delta^{-1/2}c_\lambda^t
       \prod_{j=1}^{t}\sqrt{\frac{\dlam}{s^{(j)}}}\,
       \twonorm{\xstar}.
  \label{eq:ridge-hp-det}
\end{equation}
In particular, for $\varepsilon\in(0,1)$, a common sketch size
\begin{equation}
  s\ge \frac{c_\lambda^2\dlam}{\varepsilon^2}
  \label{eq:ridge-var-sample}
\end{equation}
gives $\ex\twonorm{\xhatstar-\xstar}^2\le
\varepsilon^{2t}\twonorm{\xstar}^2$ and
$\twonorm{\xhatstar-\xstar}\le
\delta^{-1/2}\varepsilon^t\twonorm{\xstar}$ with probability at least
$1-\delta$.
\end{theorem}

Set $\eb_j=\xb^{\star(j)}-\widetilde{\xb}^{(j)}$ for $j=0,\ldots,t$.
The proof in \hyperref[proof:ridge-main]{Appendix~\ref*{app:adaptive-master}}
combines Lemmas~\ref{lem:periter} and~\ref{lem:ridge-vector-var} with
$\xb^{\star(j)}=\eb_{j-1}$ to obtain $\ex\!\left[
 \twonorm{\eb_j}^2\mid\Fcal_{j-1}\right]
 \le c_\lambda^2\frac{\dlam}{s^{(j)}}\twonorm{\eb_{j-1}}^2.$

Taking expectations and iterating this inequality with the tower property proves
\eqref{eq:ridge-ms-det}, since $\eb_0=\xstar$ and $\eb_t=\xstar-\xhatstar$.
One application of Markov's inequality gives
\eqref{eq:ridge-hp-det}. Comparing to \citet{chowdhury2018iterative}, we again save a factor of $\log d_\lambda$ on the sketch size when $c_\lambda$ is moderate.

Similarly, we provide a separate bound that replaces $c_\lambda$ by $(1-\varepsilon)^{-1}$ at the cost of a larger sketch size, which also yields an additional contraction factor of size $O(\frac{1}{\sqrt{\log d_\lambda}})$. Theorem~\ref{thm:ridgelev-struct} gives
the full statement and conditions; its proof is in
\hyperref[proof:ridge-embedding]{Appendix~\ref*{app:ridge-struct}}.

\section{Residual-Aware Sampling}
\label{sec:oracle}

The ridge leverage scores used by~\citet{chowdhury2018iterative}
depend only on the design matrix. We develop probabilities that also use the current residual solution. For a fixed
$\xb \in \R^d$ with nonzero coordinates, write
$Z(\xb)=\sum_i|\xb_i|\sqrt{w_i}$.

\begin{proposition}[Optimal sampling for the matrix product]
\label{prop:oracle}
The distribution
\begin{equation*}
 p_i^\star(\xb)=\frac{|\xb_i|\sqrt{w_i}}{Z(\xb)}
\end{equation*}
minimizes
$\ex\twonorm{\Xb^{\ts}(\Sb\Sb^{\ts}-\Ib_d)\xb}^2$.
Furthermore, at $p_i^\star(\xb), \ex\twonorm{\Xb^{\ts}(\Sb\Sb^{\ts}-\Ib_d)\xb}^2 \leq s^{-1}Z(\xb)^2.$
\end{proposition}

The proof is a direct application of Theorem 22 from~\citet{drineas2017lectures}. Now recall from  \eqref{eq:ridge-main-variance} that for ridge leverage scores, $\ex\twonorm{\Xb^{\ts}(\Sb\Sb^{\ts}-\Ib_d)\xb}^2 \leq s^{-1} \left( d_\lambda \twonorm{\xb}^2 \right)$. We introduce the following factor to quantify the advantage of oracle sampling over ridge leverage score sampling.

\begin{equation}
 \rho(\xb)=\frac{Z(\xb)^2}{\dlam\twonorm{\xb}^2}\in(0,1].
 \label{eq:rho}
\end{equation}
Cauchy--Schwarz gives $\rho(\xb)\le1$, with equality exactly when
$|\xb_i|$ is proportional to $\sqrt{w_i}$.

Surprisingly, our numerical experiments suggest that the value of $\rho(\xb)$ is very close to $2 / \pi$. We provide some theoretical justifications below.

Fix a nonzero $\xb \in \operatorname{range}(\Xb)$, and write $\xb = \Xb \hb$. 
Draw an index $i$ with probability $w_i/\dlam$ over indices with $w_i>0$, and define $\zeta=\frac{\sqrt{\dlam}\,\xb_i}
{\sqrt{w_i}\,\twonorm{\xb}}.$ Since $ \xb_i / \sqrt{w_i} = \left\langle \Xb_{i\ast} / \twonorm{\Xb_{i\ast}},\,\hb \right\rangle, $
$\zeta$ is a scaled projection of $\hb$ onto a randomly sampled row direction of $\Xb$. 
The following proposition identifies $2/\pi$ as the Gaussian reference value of $\rho(\xb)$ and quantifies the effect of departures from normality.

\begin{proposition}[Gaussian reference for $\rho$]
\label{prop:rho-gaussian}
Let $P_\zeta$ denote the distribution of $\zeta$, and define
\[
\Delta=W_1\!\left(P_\zeta,\mathcal N(0,1)\right),
\]
where $W_1$ denotes the $1$-Wasserstein distance. Then
\begin{equation*}
\left|\rho(\xb)-\frac{2}{\pi}\right|
\le 2\sqrt{\frac{2}{\pi}}\,\Delta+\Delta^2.
\end{equation*}
\end{proposition}

The proof is provided in Appendix~\ref{app:rho-gaussian}, which also reports empirical evidence that $\Delta$ is small and generally decreases with scale for the language model application presented in Section~\ref{sec:experiments}.

\textbf{Oracle convergence.}
Although $p^\star(\xb^{\star(j)})$ depends on previous sketches, it is fixed
before the next sketch is drawn. Let
$\bar\rho_t=(\prod_{j=1}^{t}\rho(\xb^{\star(j)}))^{1/t}$ be the geometric mean
of the oracle factors. Corollary~\ref{cor:oracle}, proved in
\hyperref[proof:oracle-convergence]{Appendix~\ref*{sec:oracle-thy}}, therefore gives,
for a common size $s$ and $\delta\in(0,1)$, with probability at least $1-\delta$,
\begin{equation}
 \twonorm{\xhatstar-\xstar}
 \le\delta^{-1/2}
 \left(c_\lambda\sqrt{\frac{\bar\rho_t\dlam}{s}}\right)^t
 \twonorm{\xstar}.
 \label{eq:oracle-var-scale}
\end{equation}
where $c_\lambda=(\smax^2+\lambda)/\lambda$ is defined in~\eqref{eq:clambda}.
If a fixed $\bar\rho\in(0,1]$ bounds every oracle factor at iterations with nonzero residual, then for $\varepsilon\in(0,1)$, $s\ge c_\lambda^2\bar\rho\dlam/\varepsilon^2$ yields the same $\delta^{-1/2}\varepsilon^t$ relative error guarantee as Theorem~\ref{thm:ridgelev}. We note that $\bar \rho$ is unknown before running the algorithm. An upper bound on $\bar\rho$ would be interesting future work.

\textbf{A practical approximation.}
The oracle requires the unknown residual solution. Algorithm~\ref{alg:adaptive} gives a practical approximation. It draws an initial ridge leverage sketch $\Sb^{(0)}$ of size $s_0$ and caches a
factorization of
$\Hb_0=\Ab\Sb^{(0)}\Sb^{(0)\ts}\Ab^{\ts}+\lambda\Ib_n$.
At iteration $j$, it forms
\begin{equation}
 \widehat{\xb}^{(j)}=\Ab^{\ts}\Hb_0^{-1}\bb^{(j)},
 \qquad
 \widehat p_i^{(j)}=
 \frac{|\widehat{\xb}_i^{(j)}|\sqrt{w_i}}
      {\sum_k|\widehat{\xb}_k^{(j)}|\sqrt{w_k}}.
 \label{eq:phat}
\end{equation}
This requires a solve with the cached factorization and a multiplication by
$\Ab^{\ts}$, without a new pilot factorization.
The proposal and residual are measurable with respect to
$\Fcal_{j-1}^{\mathrm{ad}}=
\sigma(\Sb^{(0)},\Sb^{(1)},\ldots,\Sb^{(j-1)})$.
For a proposal $p^{(j)}$, define its variance factor by
\begin{equation}
 R_p^{(j)}=
 \frac{\sum_i |\xb_i^{\star(j)}|^2w_i/p_i^{(j)}}
 {\dlam\twonorm{\xb^{\star(j)}}^2}.
 \label{eq:adaptive-main-factor}
\end{equation}
Ridge leverage gives $R_p^{(j)}=1$ and the oracle gives
$R_p^{(j)}=\rho(\xb^{\star(j)})$.
With $\bar R_{p,t}=(\prod_{j=1}^{t}R_p^{(j)})^{1/t}$, the adaptive guarantee
replaces $\bar\rho_t$ by $\bar R_{p,t}$ in \eqref{eq:oracle-var-scale}.
Theorem~\ref{thm:master-var} proves the general
bound in \hyperref[proof:adaptive-master]{Appendix~\ref*{app:adaptive-master}};
Corollary~\ref{cor:pure-adaptive-var} applies it to the pilot proposal in
\hyperref[proof:adaptive-pure]{Appendix~\ref*{sec:adaptive}}.
These factors can exceed one; approximation
by a pilot sketch alone does not establish contraction.

A defensive mixture with weight $\alpha\in(0,1)$ controls the worst-case performance:
\begin{equation}
 \widetilde p_i^{(j)}=(1-\alpha)\widehat p_i^{(j)}
                   +\alpha\frac{w_i}{\dlam},\qquad 0<\alpha<1.
 \label{eq:mix}
\end{equation}
Its sketching error is bounded by $1/\alpha$ times the error by ridge leverage score sampling, regardless of pilot accuracy.
Consequently, a common iteration size
$s\ge c_\lambda^2\dlam/(\alpha\varepsilon^2)$ guarantees mean square contraction
by $\varepsilon^2$ per iteration. Corollary~\ref{cor:mixture-var}, proved in
\hyperref[proof:adaptive-mixture]{Appendix~\ref*{sec:adaptive}}, gives this
guarantee and the comparison with the pure proposal.
\section{Experiments}
\label{sec:experiments}
\newcommand{\ArceneRLS}{1.41\times10^{-4}}
\newcommand{\ArceneOracle}{7.20\times10^{-6}}
\newcommand{\ArceneAdaptive}{3.31\times10^{-5}}
\newcommand{\ArceneMixture}{2.59\times10^{-5}}
\newcommand{\SyntheticRLS}{4.48\times10^{-3}}
\newcommand{\SyntheticOracle}{1.49\times10^{-3}}
\newcommand{\SyntheticAdaptive}{1.10\times10^{-2}}
\newcommand{\SyntheticMixture}{7.77\times10^{-3}}
\newcommand{\ArceneReused}{1.85\times10^{-3}}
\newcommand{\ArceneFresh}{1.52\times10^{-13}}
\newcommand{\QwenLambda}{10^{4.5}}
\newcommand{\QwenDof}{26.05}
\newcommand{\QwenRhoInitial}{0.635}
\newcommand{\QwenBudgetRidge}{3.02\times10^{-10}}
\newcommand{\QwenBudgetOracle}{3.42\times10^{-11}}
\newcommand{\QwenBudgetAdaptive}{9.76\times10^{-11}}
\newcommand{\QwenBudgetMixture}{6.19\times10^{-11}}
\newcommand{\QwenAdaptiveFactorMedian}{0.72}
\newcommand{\QwenAdaptiveFactorMin}{0.660}
\newcommand{\QwenAdaptiveFactorMax}{12.56}
\newcommand{\QwenMixtureFactorMedian}{0.66}
\newcommand{\QwenMixtureFactorMin}{0.617}
\newcommand{\QwenMixtureFactorMax}{0.695}

We compare uniform, leverage, ridge leverage, oracle, adaptive, and mixture
sampling on ARCENE, synthetic data, and a ridge probe built from saved
Qwen2.5 representations. The ARCENE and synthetic experiments extend the
experiments of \citet{chowdhury2018iterative}. All six methods use the same residual update with fresh sketches; a separate experiment compares fixed and
fresh sketches (Appendix~\ref{app:baseline-settings}).
Each dataset uses the same six panels: (a) solution error, (b) objective
excess, (c) a sketch size sweep, (d) a regularization sweep, (e) the oracle
factor along each method's residual trajectory, and (f) the actual proposal
factor for adaptive sampling and its mixture. Panel (d) displays
regularization through its effective dimension $\dlam$.

Curves show medians and interquartile ranges across independent sketch
trials. We report relative error
$\twonorm{\xhatstar-\xstar}/\twonorm{\xstar}$ and objective excess
$(f(\xhatstar)-f(\xstar))/f(\xstar)$, evaluated stably as described in
Appendix~\ref{app:implementation}. The displayed budget $s$ is the baseline
number of columns sketched per update. For a run of $t$ updates, adaptive methods
spend $s$ columns on an initial pilot and $\lfloor s(t-1)/t\rfloor$ per
update, so their total is at most $st$. The mixture uses ridge weight $0.1$.

\textbf{Comparison with the leverage baselines.}
ARCENE \citep{guyon2004result} gives a $200\times10000$ design, scaled
globally to $[0,1]$, with labels $\pm1$. We use $15$ trials and
$\lambda=10$ except in the regularization sweep.
In Figure~\ref{fig:arcene}(a,b), the oracle reaches numerical precision significantly faster than the 3 baselines. Adaptive and mixture methods behave very similarly to the orcale method. Their lower errors also persist across the sketch size and regularization sweeps in panels (c,d). We additionally perform the same experiments on the synthetic data used in \citet{chen2015fast}, which has $n=500$, $d=50000$, and we use 5 trials. The results are similar to ARCENE except that the gap between the 3 baselines becomes smaller. See the full result at Figure~\ref{fig:synth}.

A separate fixed versus fresh comparison shows the benefit of refreshing sketches (Figure~\ref{fig:fresh}). At the two smaller budgets, median errors eventually grow for every reused-sketch baseline. At the smallest budget, fresh uniform and leverage sampling also show growing median errors, and fresh ridge leverage makes little progress. The residual-aware methods reduce errors at every tested budget. Appendix~\ref{app:baseline-settings}
gives the full settings and additional plots.

\begin{figure}[t]
 \centering
 \includegraphics[width=\linewidth]{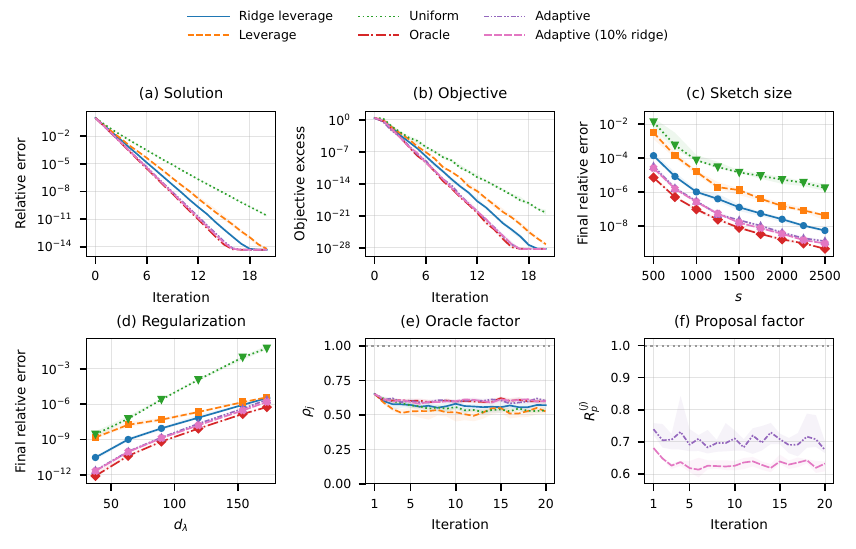}
 \caption{ARCENE, $15$ trials. Panels (a,b,e,f) use 20 updates at
 $s=2400$, $\lambda=10$. Panels (c,d) use 10 updates: (c) varies $s$ at
 $\lambda=10$; (d) varies $\lambda$ at $s=2400$. Bands show interquartile
 ranges. Adaptive budgets include the pilot. Panel (e) shows all six
 methods; (f) shows adaptive sampling and its mixture.
 }
 \label{fig:arcene}
\end{figure}

\begin{figure}[htbp]
 \centering
 \includegraphics[width=\linewidth]{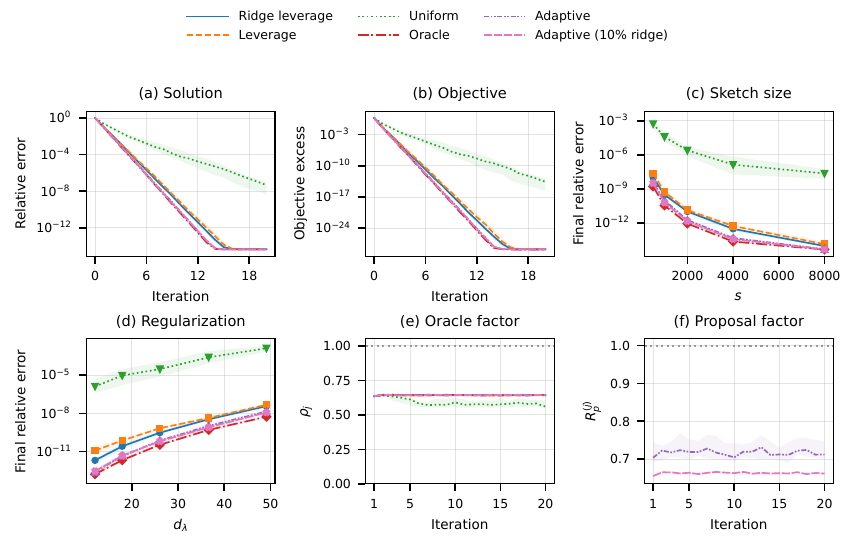}
 \caption{Qwen2.5-32B probe with $102400$ features, $30$ trials. Panels (a,b,e,f)
 use $s=1000$ and 20 updates; (c) varies $s$ with 10 updates.
 These panels use the cross-validated $\lambda=\QwenLambda$. Panel (d) varies
 $\lambda$ at $s=1000$ with 10 updates. Bands show interquartile ranges.
 Adaptive budgets include the pilot. Panel (e) shows all six methods;
 (f) shows adaptive sampling and its mixture.}
 \label{fig:qwen}
\end{figure}

\textbf{Ridge probing of language model representations.}
We fit a ridge probe to frozen Qwen2.5-32B-Instruct representations
\citep{qwen2025technical} for semantic textual similarity
\citep{cer2017stsb}, using 100 training examples.
Concatenating representations from 20 layers at the final retained
non-padding token gives $102{,}400$ features per example.
We account for an unpenalized intercept by centering features and
labels using training means.
The regularization parameter $\lambda$ is selected by 5-fold
cross-validation on the training set.
Appendix~\ref{app:qwen-settings} provides details of feature construction
and parameter selection.

The results of Figure~\ref{fig:qwen} resembel Figure~\ref{fig:arcene} except that uniform sampling converges significantly slower. We again observe the superior performance of residual-aware sampling over the baselines.

\textbf{Observed variance factors.}
Before every update, we record the oracle factor
$\rho_j=\rho(\xb^{\star(j)})$ along each method's residual trajectory
and the actual proposal factor $R_p^{(j)}$ from
\eqref{eq:adaptive-main-factor}. Panels (e,f) use the same trials and all 20 updates as (a,b). Panel (e) shows the best attainable
directional factor for each method's current residual; panel (f) evaluates
the probabilities used by adaptive sampling and its mixture.
Across the three datasets, the mixture has lower median proposal factors
and generally narrower interquartile bands than pure adaptive sampling.
The ridge leverage reference $R_p=1$ is shown as a horizontal line.

\section{Conclusion}
\label{sec:conclusion}
We establish sharper convergence bounds for fresh sampling in the iterative
column sampling framework of \citet{chowdhury2018iterative} for ridge regression.
Building on the analysis of \citet{kacham2022sketching}, we analyze the variance
of the sketched matrix-vector product and exploit conditional independence
across iterations. This analysis yields an oracle sampling rule that minimizes
the variance, along with practical adaptive and mixture strategies that
approximate the oracle. Experiments on synthetic and real data, including
ridge probes on pretrained Qwen2.5 hidden layers, demonstrate gains in solution
accuracy.

A limitation of our analysis is that the sharper convergence guarantees for
oracle, adaptive, and mixture sampling depend on factors determined by the
residual trajectory and, for the practical methods, the quality of the sampling
proposals. These factors are unknown before the algorithm runs. Although the
oracle factor is at most one and the defensive mixture bounds its proposal
factor by $1/\alpha$, these worst-case bounds do not quantify the improvements
observed in our experiments. An open question is to identify conditions on the
problem and the pilot approximation that yield sharper a priori bounds on
these factors or their products across iterations. Such bounds could clarify
when these sampling rules provably improve upon ridge leverage sampling and
guide the choice of sketch sizes and mixture weights.

\label{sec:main-end}

\newpage
\bibliographystyle{abc2027_conference}
\bibliography{abc2027_conference}
\clearpage
\appendix

\section{Algorithms}
\label{app:algorithms}

\begin{algorithm}[htbp]
\caption{Sampling and rescaling \citep{chowdhury2018iterative}}
\label{alg:sample}
\begin{algorithmic}[1]
\STATE \textbf{Input:} Distribution $p$ on $\{1,\ldots,d\}$; integer $s\ge1$.
\STATE Initialize $\Sb=\zero_{d\times s}$.
\FOR{$k=1,\ldots,s$}
\STATE Independently draw $i_k$ with probability $p_{i_k}$.
\STATE Set $\Sb_{i_k,k}=1/\sqrt{s p_{i_k}}$.
\ENDFOR
\STATE \textbf{Return:} $\Sb$.
\end{algorithmic}
\end{algorithm}

\begin{algorithm}[htbp]
\caption{Adaptive sampling with a cached pilot factorization}
\label{alg:adaptive}
\begin{algorithmic}[1]
\STATE \textbf{Input:} $\Ab,\bb$; regularization $\lambda>0$; $t\ge1$; positive integer sizes
$s_0,s^{(1)},\ldots,s^{(t)}$; mixture weight $\alpha\in[0,1)$.
\STATE Draw a pilot sketch $\Sb^{(0)}$ of size $s_0$ using Algorithm~\ref{alg:sample} with $p_i = w_i / d_\lambda$. Factor
$\Hb_0=\Ab\Sb^{(0)}\Sb^{(0)\ts}\Ab^{\ts}+\lambda\Ib_n$.
\STATE Initialize $\bb^{(0)}=\bb$, $\widetilde{\xb}^{(0)}=\zero_d$, and $\yb^{(0)}=\zero_n$.
\FOR{$j=1,\ldots,t$}
\STATE Set $\bb^{(j)}=\bb^{(j-1)}-\lambda\yb^{(j-1)}-\Ab\widetilde{\xb}^{(j-1)}$.
\STATE Solve $\Hb_0\ub=\bb^{(j)}$ using the cached factorization; set
$\widehat{\xb}^{(j)}=\Ab^{\ts}\ub$.
\STATE Let $z=\sum_i|\widehat{\xb}_i^{(j)}|\sqrt{w_i}$.
\STATE Set $\widehat p_i^{(j)}=|\widehat{\xb}_i^{(j)}|\sqrt{w_i}/z$.
\STATE Set $p_i^{(j)}=(1-\alpha)\widehat p_i^{(j)}+\alpha w_i/\dlam$.
\STATE Draw a fresh $\Sb^{(j)}$ of size $s^{(j)}$ using Algorithm~\ref{alg:sample} with $p^{(j)}$.
\STATE Form $\Hb_j=\Ab\Sb^{(j)}\Sb^{(j)\ts}\Ab^{\ts}+\lambda\Ib_n$ and solve
$\Hb_j\yb^{(j)}=\bb^{(j)}$.
\STATE Set $\widetilde{\xb}^{(j)}=\Ab^{\ts}\yb^{(j)}$.
\ENDFOR
\STATE \textbf{Return:} $\xhatstar=\sum_{j=1}^{t}\widetilde{\xb}^{(j)}$.
\end{algorithmic}
\end{algorithm}

\section{Running Time Analysis}

We discuss the running time of Algorithm~\ref{alg:fresh} with LS or RLS, assuming the exact leverage scores or ridge leverage scores are provided. Approximations can be computed efficiently~\citep{drineas2012leverage,clarkson2013lowrank,cohen2017ridge}. Within one iteration of~\eqref{eq:iteration}, the two products $\Ab\widetilde{\xb}^{(j-1)}$ and $\Ab^{\ts}\yb^{(j)}$ cost $O(\nnz{\Ab})$, constructing $\Ab\Sb^{(j)}$ costs $O(ns)$, forming $\Hb_j$ costs $O(n^2s)$, and factoring and solving the resulting $n\times n$ system costs $O(n^3)$. The remaining vector operations are dominated by these costs. Thus each iteration costs $O(\nnz{\Ab}+n^2\max(n,s))$.

\paragraph{Relative solution error.}
Fix $\varepsilon,\varepsilon_{\texttt{alg}},\delta\in(0,1)$, where $\varepsilon$ is the per-iteration contraction parameter. By Theorems~\ref{thm:colev-ms} and~\ref{thm:ridgelev}, choosing the common sketch sizes
\begin{equation}
 s_{\mathrm{LS}}=\left\lceil
   \left(\frac{\smax^2}{\lambda}\right)^2\frac{n-1}{\varepsilon^2}
 \right\rceil,
 \qquad
 s_{\mathrm{RLS}}=\left\lceil\frac{c_\lambda^2\dlam}{\varepsilon^2}\right\rceil
 \label{eq:runtime-sketch-sizes}
\end{equation}
with $c_\lambda=(\smax^2+\lambda)/\lambda$ from~\eqref{eq:clambda}
gives $\twonorm{\xhatstar-\xstar}\le \delta^{-1/2}\varepsilon^t\twonorm{\xstar}$ with probability at least $1-\delta$. Consequently, $\twonorm{\xhatstar-\xstar}\le \varepsilon_{\texttt{alg}}\twonorm{\xstar}$ is guaranteed after
\begin{equation*}
 t_{\mathrm{sol}}=\left\lceil
 \frac{\log(1/\varepsilon_{\texttt{alg}})+\tfrac12\log(1/\delta)}
      {\log(1/\varepsilon)}
 \right\rceil
\end{equation*}
iterations. With the sketch sizes in~\eqref{eq:runtime-sketch-sizes}, the running times for $t$ iterations are
\begin{align}
 T_{\mathrm{LS}}(t)
 &=O\!\left(\nnz{\Ab}
   +\frac{n^3}{\varepsilon^2}\left(\frac{\smax^2}{\lambda}\right)^2
   \right)t,
 \label{eq:runtime-ls}\\
 T_{\mathrm{RLS}}(t)
 &=O\!\left(\nnz{\Ab}+n^3
   +\frac{n^2c_\lambda^2\dlam}{\varepsilon^2}\right)t.
 \label{eq:runtime-rls}
\end{align}
Here the LS simplification assumes $\lambda\le\smax^2$.

\paragraph{Relative objective error.}
Let $f(\xb)=\twonorm{\Ab\xb-\bb}^2+\lambda\twonorm{\xb}^2$ and $\eb=\xhatstar-\xstar$. The optimality condition $\Ab^{\ts}(\Ab\xstar-\bb)+\lambda\xstar=\zero_d$ cancels the cross terms, giving
\[
 f(\xhatstar)-f(\xstar)
 =\twonorm{\Ab\eb}^2+\lambda\twonorm{\eb}^2
 \le(\smax^2+\lambda)\twonorm{\eb}^2.
\]
Since $f(\xstar)\ge\lambda\twonorm{\xstar}^2$ and $c_\lambda=(\smax^2+\lambda)/\lambda$, for $\bb\ne\zero_n$ we obtain
\begin{equation*}
 \frac{f(\xhatstar)-f(\xstar)}{f(\xstar)}
 \le c_\lambda\frac{\twonorm{\xhatstar-\xstar}^2}{\twonorm{\xstar}^2}
 \le\frac{c_\lambda}{\delta}\varepsilon^{2t}
\end{equation*}
with probability at least $1-\delta$, with the sketch sizes in~\eqref{eq:runtime-sketch-sizes}. If $\bb=\zero_n$, the solution and all iterates are zero, so the objective guarantee is trivial. To ensure $f(\xhatstar)\le(1+\varepsilon_{\texttt{alg}})f(\xstar)$ with probability at least $1-\delta$, it suffices to take
\begin{equation*}
 t_{\mathrm{obj}}=\left\lceil
 \frac{\log\!\bigl(c_\lambda/(\delta\varepsilon_{\texttt{alg}})\bigr)}
      {2\log(1/\varepsilon)}
 \right\rceil.
\end{equation*}
The corresponding running times are $T_{\mathrm{LS}}(t_{\mathrm{obj}})$ and $T_{\mathrm{RLS}}(t_{\mathrm{obj}})$ from~\eqref{eq:runtime-ls}--\eqref{eq:runtime-rls}, using the same sketch sizes as in~\eqref{eq:runtime-sketch-sizes}.

\section{Related work}
\label{sec:related}

\paragraph{Sketching for ridge regression.}
\citet{chen2015fast} give a fast relative-error approximation algorithm for
ridge regression. \citet{wang2017sketched} study the statistical and optimization impacts of classical sketch and Hessian sketch used for matrix ridge regression problems. \citet{avron2017sharper} derive bounds for regularized objective
approximation governed by effective dimension rather than rank.
\citet{chowdhury2018iterative} introduce an iterative solver using leverage
and ridge leverage sampling, and ask whether fresh randomness improves
convergence. \citet{kacham2022sketching} prove geometric convergence for
the refreshed solver and improve sparsity guarantees for oblivious sketches.
Their Lemma~3.3 controls sketch error in the direction of the exact residual
solution, the same matrix-vector product analyzed in this paper. 
\paragraph{Iterative sketching and preconditioning.}
\citet{pilanci2016iterative} introduced the Iterative Hessian Sketch (IHS) for convex-constrained least-squares problems, using fresh random sketches to approximate the Hessian at each iteration and achieving geometric convergence with sketch dimensions governed by the intrinsic complexity of the constraint set.
\citet{wang2017dual} connect Hessian sketching to dual random projection
and develop accelerated methods using preconditioned conjugate gradients.
Sketching also provides preconditioners for ridge regression with
stochastic variance reduced gradient \citep{gonen2016solving} and for
kernel ridge systems through random features \citep{avron2017faster}.
\citet{lacotte2019faster} compare fixed and refreshed sketches used in iterative methods for solving overdetermined least-squares problems, where the sketches are data-oblivious. They give a closed-form formula for the expected error of IHS. \citet{lacotte2021newton} consider choosing the sketch sizes adaptively.

\paragraph{Feature and column sampling for ridge regression.}
Leverage scores are used in least-squares regression
\citep{drineas2006regression} and low-rank matrix approximation via
column and row sampling \citep{drineas2006subspace,drineas2008cur,mahoney2009cur}.
Ridge leverage scores connect
sampling to effective dimension in kernel ridge regression
\citep{alaoui2015fast} and support low-rank and Nystr\"om approximation
\citep{cohen2017ridge,musco2017recursive}.
\citet{li2024optimal} use sampling probabilities similar to our oracle
proposal: applied to the residual problem, their ideal rule is
$p_i^{\mathrm{LN},(j)} \propto |\xb_i^{\star(j)}|\twonorm{\Ab_{\ast i}}$,
while ours is $p_i^{\star(j)} \propto |\xb_i^{\star(j)}|\sqrt{w_i}$. Complementarily,~\citet{kacham20a} study deterministic coresets for multi-response ridge regression, obtaining coresets of size $O((d_\lambda+d')/\varepsilon)$, together with corresponding lower bounds, where $d'$ denotes the number of response vectors. Their row-sampling coreset setting, however, is not directly comparable to our column-sampling framework.

\paragraph{Ridge probes for language models.}
Probing fits simple predictors to a model's internal representations
\citep{belinkov2022probing}. Ridge probes have been used to decode spatial and
temporal coordinates \citep{gurnee2024space} and numerical values
\citep{zhu2025numbers}. Our application concerns computing a probe on fixed
hidden features, with semantic similarity as the prediction target.

\section{Preliminaries and embedding bounds}
\label{app:preliminaries}

\subsection{An explicit embedding tail}
\label{app:embedding-tail}
Spectral approximation of matrix products can be controlled through stable
rank \citep{cohen2016optimal}. For the column sampling laws studied here,
matrix concentration gives spectral error bounds for sampled Gram matrices
\citep{tropp2015introduction}. We use the explicit specialization in the
proof of \citet[Theorem~3]{chowdhury2018iterative}, retaining its constants
and admissibility condition below.
\begin{lemma}[Explicit failure probability]
\label{lem:cyd-struct-tail}
Let $\Xb\in\RR{d}{n}$ be nonzero with $\twonorm{\Xb}\le1$, and let $\Sb$ be
constructed by Algorithm~\ref{alg:sample} with
$p_i=\twonorm{\Xb_{i\ast}}^{2}/\|\Xb\|_{F}^{2}$. Fix
$\varepsilon\in(0,1]$. If
\begin{equation}
  s\ge\frac{2\|\Xb\|_{F}^{2}}{\varepsilon^{2}},
  \label{eq:cyd-tail-size}
\end{equation}
then
\begin{align}
  \pr{\twonorm{\Xb^{\ts}\Sb\Sb^{\ts}\Xb-\Xb^{\ts}\Xb}>\varepsilon}
  &\le
  4(1+\|\Xb\|_{F}^{2})
  \exp\!\left(
    -\frac{s\varepsilon^{2}}
    {\|\Xb\|_{F}^{2}(2+2\varepsilon/3)}
  \right) \label{eq:cyd-tail-exact}\\
  &\le
  4(1+\|\Xb\|_{F}^{2})
  \exp\!\left(-\frac{3s\varepsilon^{2}}{8\|\Xb\|_{F}^{2}}\right).
  \label{eq:cyd-tail-simple}
\end{align}
\end{lemma}

In the concentration argument cited below, $\rho_1$ denotes the uniform
summand bound, and $\Pb$ denotes the positive semidefinite variance bound.

\begin{proof}
In the proof of \citet[Theorem~3]{chowdhury2018iterative}, equations~(119)
and~(124) give
$\rho_{1}=\|\Xb\|_{F}^{2}/s$ and
$\|\Pb\|_{2}=\|\Xb\|_{F}^{2}/s$, respectively. Thus the required
admissibility condition is
\begin{equation}
  \varepsilon\ge
  \sqrt{\frac{\|\Xb\|_{F}^{2}}{s}}
  +\frac{\|\Xb\|_{F}^{2}}{3s}.
  \label{eq:cyd-tail-admissibility}
\end{equation}
Under \eqref{eq:cyd-tail-size}, the right-hand side of
\eqref{eq:cyd-tail-admissibility} is at most
\[
  \frac{\varepsilon}{\sqrt{2}}+\frac{\varepsilon^{2}}{6}
  \le\left(\frac{1}{\sqrt{2}}+\frac{1}{6}\right)\varepsilon
  <\varepsilon,
\]
so the condition is satisfied. Equation~(128) of that proof now gives
\eqref{eq:cyd-tail-exact}. Since $\varepsilon\le1$, we have
$2+2\varepsilon/3\le8/3$, which proves \eqref{eq:cyd-tail-simple}.
\end{proof}

\subsection{Residual identities}
\label{app:residual-identities}
\begin{lemma}[Residual decomposition]
\label{lem:chow-decomp}\label{lem:chow-rec}\label{lem:chow-acc}
For the iteration in \eqref{eq:iteration}, any sketches, and $j=1,\ldots,t$,
\begin{align}
 \xb^{\star(j)}&=\Vb\Slam^2\Sigmab^{-1}\Ub^{\ts}\bb^{(j)},
 \label{eq:residual-svd}\\
 \xb^{\star(j)}&=\xb^{\star(j-1)}-\widetilde{\xb}^{(j-1)},
 \label{eq:residual-rec}\\
 \xstar&=\xb^{\star(t)}+\sum_{j=1}^{t-1}\widetilde{\xb}^{(j)},\qquad
 \xhatstar-\xstar=\widetilde{\xb}^{(t)}-\xb^{\star(t)}.
 \label{eq:residual-decomp}
\end{align}
\end{lemma}
\begin{proof}
\phantomsection\label{proof:residual}
These are the identities used by \citet{chowdhury2018iterative}.
Substitution of the SVD into the exact residual solution gives
\eqref{eq:residual-svd}. Apply
$\Ab^{\ts}(\Ab\Ab^{\ts}+\lambda\Ib_n)^{-1}$ to the residual update
in \eqref{eq:iteration} to obtain \eqref{eq:residual-rec}. Telescoping from
$\xb^{\star(0)}=\xstar$, with $\widetilde{\xb}^{(0)}=\zero_d$, gives
\eqref{eq:residual-decomp}.
\end{proof}

\subsection{Fixed sketch comparison}
\label{app:fixed-comparison}
For completeness, the two embedding conditions in
\citet[Theorems~1 and~2]{chowdhury2018iterative} are
\begin{align*}
 \twonorm{\Vb^{\ts}\Sb\Sb^{\ts}\Vb-\Ib_n}&\le\varepsilon/2,\\
 \twonorm{\Slam\Vb^{\ts}\Sb\Sb^{\ts}\Vb\Slam-\Slam^2}
 &\le\varepsilon/(4\sqrt{2}).
\end{align*}
For $\varepsilon\in(0,1)$, the first gives
\begin{equation*}
 \twonorm{\xhatstar-\xstar}\le\varepsilon^t\twonorm{\xstar},
\end{equation*}
and the second gives
\begin{equation*}
 \twonorm{\xhatstar-\xstar}\le\frac{\varepsilon^t}{2}
 \left(\twonorm{\xstar}+
 \frac{\twonorm{\Ub_{k,\perp}^{\ts}\bb}}{\sqrt{2\lambda}}\right).
\end{equation*}
Here $k \in \{1,\cdots,n \}$ counts singular values with $\sigma_i^2>\lambda$, and $\Ub_{k,\perp}$ contains the remaining left singular vectors. Leverage and ridge leverage sampling respectively satisfy these events with probability at least $1-\delta$, for
$\delta\in(0,1)$, with sufficient
sample sizes of order $O(\varepsilon^{-2} n \ln n)$ and $O(\varepsilon^{-2} d_\lambda \ln d_\lambda)$.
One sketch is reused for all iterations in these bounds. 

\section{Additional Results for Leverage Sampling}
\label{app:colev-proofs}

Throughout this section, sketches use independent sampling with replacement
with $p_i=\twonorm{\Vb_{i\ast}}^2/n$ and deterministic sizes $s^{(j)}$.
Rescaling gives
$\ex[\Vb^{\ts}\Sb^{(j)}\Sb^{(j)\ts}\Vb]=\Ib_n$ and therefore
$\ex[\Ehat^{(j)}]=\zero$.

\subsection{Error Identities and Vector Variance}
\label{sec:telescope}
\label{app:telescope}

We first record the commutation identity used to telescope the residuals.
\begin{lemma}[Commutation identity]
\label{lem:commute}
\label{app:commute}
Define $\widetilde{\Rb}^{(j)}=\Ib_n-
(\Ib_n+\Ehat^{(j)}\Gb^{-1})^{-1}$. Then
\begin{equation}
  \Gb^{-1}\widetilde{\Rb}^{(j)}
  =-\Rhat^{(j)}\Gb^{-1}.
  \label{eq:commute}
\end{equation}
\end{lemma}
\begin{proof}
Both $\Ib_{n}+\Gb^{-1}\Ehat^{(j)}$ and $\Ib_{n}+\Ehat^{(j)}\Gb^{-1}$ are invertible, and
\begin{equation*}
  \Gb(\Ib_{n}+\Gb^{-1}\Ehat^{(j)})
  =\Gb+\Ehat^{(j)}
  =(\Ib_{n}+\Ehat^{(j)}\Gb^{-1})\Gb .
\end{equation*}
Inverting this identity gives 
\begin{equation*}
  (\Ib_{n}+\Gb^{-1}\Ehat^{(j)})^{-1}\Gb^{-1}
  =\Gb^{-1}(\Ib_{n}+\Ehat^{(j)}\Gb^{-1})^{-1}.
\end{equation*}
Hence, using the closed forms \eqref{eq:R} and the definition in Lemma~\ref{lem:commute},
\begin{align*}
  \Gb^{-1}\widetilde{\Rb}^{(j)}
  &=\Gb^{-1}\big[\Ib_{n}-(\Ib_{n}+\Ehat^{(j)}\Gb^{-1})^{-1}\big]
   =\Gb^{-1}-\Gb^{-1}(\Ib_{n}+\Ehat^{(j)}\Gb^{-1})^{-1}\\
  &=\Gb^{-1}-(\Ib_{n}+\Gb^{-1}\Ehat^{(j)})^{-1}\Gb^{-1}
   =-\big[(\Ib_{n}+\Gb^{-1}\Ehat^{(j)})^{-1}-\Ib_{n}\big]\Gb^{-1}
   =-\Rhat^{(j)}\Gb^{-1},
\end{align*}
which is \eqref{eq:commute}.
\end{proof}

\begin{proof}[\textbf{Proof of Lemma~\ref{lem:telescope}}]
\phantomsection\label{proof:column-product}
By Lemma~\ref{lem:chow-decomp}, $\xhatstar-\xstar=\widetilde{\xb}^{(t)}-\xb^{\star(t)}$,
and by \citet[Lemma 11]{chowdhury2018iterative} this equals
$\Vb\Rhat^{(t)}\Gb^{-1}\Sigmab^{-1}\Ub^{\ts}\bb^{(t)}$. We now unwind $\bb^{(t)}$.
For $j=2,\ldots,t$, the updates of Algorithm~\ref{alg:fresh} give
$\bb^{(j)}=\bb^{(j-1)}-(\Ab\Ab^{\ts}+\lambda\Ib_{n})
(\Ab\Sb^{(j-1)}\Sb^{(j-1)\ts}\Ab^{\ts}+\lambda\Ib_{n})^{-1}\bb^{(j-1)}$ (collecting
$\lambda\yb^{(j-1)}+\Ab\widetilde{\xb}^{(j-1)}=(\Ab\Ab^{\ts}+\lambda\Ib_{n})\yb^{(j-1)}$).
Using the thin SVD and $\Ub^{\ts}\Ub=\Ib_{n}$,
\begin{align*}
  \Ab\Ab^{\ts}+\lambda\Ib_{n}&=\Ub\Sigmab(\Ib_{n}+\lambda\Sigmab^{-2})\Sigmab\Ub^{\ts}
  =\Ub\Sigmab\Gb\Sigmab\Ub^{\ts},\\
  \Ab\Sb^{(j-1)}\Sb^{(j-1)\ts}\Ab^{\ts}+\lambda\Ib_{n}
  &=\Ub\Sigmab\big(\Vb^{\ts}\Sb^{(j-1)}\Sb^{(j-1)\ts}\Vb+\lambda\Sigmab^{-2}\big)\Sigmab\Ub^{\ts}\\
  &=\Ub\Sigmab(\Ehat^{(j-1)}+\Gb)\Sigmab\Ub^{\ts},
\end{align*}
where the last equality uses $\Ehat^{(j-1)}+\Gb=(\Vb^{\ts}\Sb^{(j-1)}\Sb^{(j-1)\ts}\Vb-\Ib_{n})
+(\Ib_{n}+\lambda\Sigmab^{-2})$. Hence
\begin{align*}
  (\Ab\Ab^{\ts}+\lambda\Ib_{n})(\Ab\Sb^{(j-1)}\Sb^{(j-1)\ts}\Ab^{\ts}+\lambda\Ib_{n})^{-1}
  &=\Ub\Sigmab\Gb\Sigmab\Ub^{\ts}\cdot\Ub\Sigmab^{-1}(\Ehat^{(j-1)}+\Gb)^{-1}\Sigmab^{-1}\Ub^{\ts}\\
  &=\Ub\Sigmab\Gb(\Gb+\Ehat^{(j-1)})^{-1}\Sigmab^{-1}\Ub^{\ts}.
\end{align*}
Factor $\Gb+\Ehat^{(j-1)}=(\Ib_{n}+\Ehat^{(j-1)}\Gb^{-1})\Gb$, so
$\Gb(\Gb+\Ehat^{(j-1)})^{-1}=\Gb\Gb^{-1}(\Ib_{n}+\Ehat^{(j-1)}\Gb^{-1})^{-1}
=(\Ib_{n}+\Ehat^{(j-1)}\Gb^{-1})^{-1}$. Therefore
\begin{align*}
  \bb^{(j)}
  &=\bb^{(j-1)}-\Ub\Sigmab(\Ib_{n}+\Ehat^{(j-1)}\Gb^{-1})^{-1}\Sigmab^{-1}\Ub^{\ts}\bb^{(j-1)}\\
  &=\Ub\Sigmab\big[\Ib_{n}-(\Ib_{n}+\Ehat^{(j-1)}\Gb^{-1})^{-1}\big]\Sigmab^{-1}\Ub^{\ts}\bb^{(j-1)}
  =\Ub\Sigmab\,\widetilde{\Rb}^{(j-1)}\Sigmab^{-1}\Ub^{\ts}\bb^{(j-1)},
\end{align*}
using $\bb^{(j-1)}=\Ub\Sigmab\,\Sigmab^{-1}\Ub^{\ts}\bb^{(j-1)}$ and
$\Ib_{n}-(\Ib_{n}+\Ehat^{(j-1)}\Gb^{-1})^{-1}=\widetilde{\Rb}^{(j-1)}$, which is the
definition in Lemma~\ref{lem:commute}.
Premultiplying by $\Sigmab^{-1}\Ub^{\ts}$ and using $\Ub^{\ts}\Ub=\Ib_{n}$,
$\Sigmab^{-1}\Sigmab=\Ib_{n}$,
\begin{equation}
  \Sigmab^{-1}\Ub^{\ts}\bb^{(j)}=\widetilde{\Rb}^{(j-1)}\Sigmab^{-1}\Ub^{\ts}\bb^{(j-1)}.
  \label{eq:subrec}
\end{equation}
Apply \eqref{eq:subrec} repeatedly, starting from
$\xhatstar-\xstar=\Vb\Rhat^{(t)}\Gb^{-1}\Sigmab^{-1}\Ub^{\ts}\bb^{(t)}$ and ending at
$\bb^{(1)}=\bb^{(0)}=\bb$:
\begin{equation*}
  \xhatstar-\xstar
  =\Vb\Rhat^{(t)}\Gb^{-1}\,\widetilde{\Rb}^{(t-1)}\cdots\widetilde{\Rb}^{(1)}\,
  \Sigmab^{-1}\Ub^{\ts}\bb.
\end{equation*}
Finally, move $\Gb^{-1}$ rightward through the $(t-1)$ factors $\widetilde{\Rb}$
using $\Gb^{-1}\widetilde{\Rb}^{(j)}=-\Rhat^{(j)}\Gb^{-1}$ from \eqref{eq:commute};
each pass flips a sign, producing $(-1)^{t-1}$:
\begin{equation*}
  \Gb^{-1}\widetilde{\Rb}^{(t-1)}\cdots\widetilde{\Rb}^{(1)}
  =(-1)^{t-1}\Rhat^{(t-1)}\cdots\Rhat^{(1)}\Gb^{-1}.
\end{equation*}
Substituting gives
$\xhatstar-\xstar=(-1)^{t-1}\Vb\Rhat^{(t)}\Rhat^{(t-1)}\cdots\Rhat^{(1)}\Gb^{-1}\Sigmab^{-1}\Ub^{\ts}\bb$,
which is \eqref{eq:telescope}.
The terminal expression equals that in \eqref{eq:telescope} because
$\Gb^{-1}\Sigmab^{-1}\Ub^{\ts}\bb=\Vb^{\ts}\xstar$.

\end{proof}

The terminal vector can also be written as
\begin{equation}
  \Vb^{\ts}\xstar
  =\Gb^{-1}\Sigmab^{-1}\Ub^{\ts}\bb,
  \qquad \twonorm{\Vb^{\ts}\xstar}=\twonorm{\xstar}.
  \label{eq:tail-is-xstar}
\end{equation}
Thus, for any sketches satisfying $\twonorm{\Rhat^{(j)}}\le\varepsilon$,
\begin{equation*}
  \twonorm{\xhatstar-\xstar}\le\varepsilon^t\twonorm{\xstar}.
\end{equation*}
No independence assumption is required for these identities.

\begin{proof}[\textbf{Proof of Lemma~\ref{lem:colev-vector-var}}]
\phantomsection\label{proof:column-variance}
Write $\vb_{i}^{\ts}:=\Vb_{i\ast}$. If
$i_{1},\dots,i_{s^{(j)}}$ are the sampled indices for $\Sb^{(j)}$, then
\[
  \Ehat^{(j)}\zb
  =\frac{1}{s^{(j)}}\sum_{\ell=1}^{s^{(j)}}
  \frac{1}{p_{i_{\ell}}}\vb_{i_{\ell}}\vb_{i_{\ell}}^{\ts}\zb-\zb .
\]
Writing $\wb_{\ell}^{(j)}:=p_{i_{\ell}}^{-1}\vb_{i_{\ell}}\vb_{i_{\ell}}^{\ts}\zb$
for the summands, these are i.i.d.\ with mean
\[
  \ex[\wb_{1}^{(j)}]
  =\sum_{i=1}^{d}p_i\,p_i^{-1}\vb_i\vb_i^{\ts}\zb
  =\Big(\sum_{i=1}^{d}\vb_i\vb_i^{\ts}\Big)\zb
  =\Vb^{\ts}\Vb\,\zb=\zb,
\]
where $\sum_{i}\vb_i\vb_i^{\ts}=\Vb^{\ts}\Vb=\Ib_{n}$ because $\Vb$ has orthonormal
columns. Their second moment is
\begin{align*}
  \ex\!\left[\twonorm{\wb_{1}^{(j)}}^{2}\right]
  &=\sum_{i=1}^{d}p_i\,\frac{1}{p_i^{2}}
    \twonorm{\vb_i}^{2}(\vb_i^{\ts}\zb)^{2}\\
  &=\sum_{i=1}^{d}\frac{1}{p_i}
    \twonorm{\vb_i}^{2}(\vb_i^{\ts}\zb)^{2}\\
  &=n\sum_{i=1}^{d}(\vb_i^{\ts}\zb)^{2}
   =n \twonorm{\Vb \zb}^2 =n\twonorm{\zb}^{2}.
\end{align*}
Since
$\Ehat^{(j)}\zb=\frac{1}{s^{(j)}}\sum_{\ell=1}^{s^{(j)}}(\wb_{\ell}^{(j)}-\zb)$
is an average of i.i.d.\ mean-zero vectors, the cross terms vanish, so
\[
  \ex\!\left[\twonorm{\Ehat^{(j)}\zb}^{2}\right]
  =\frac{1}{s^{(j)}}\,\ex\!\left[\twonorm{\wb_{1}^{(j)}-\zb}^{2}\right]
  =\frac{1}{s^{(j)}}\Big(\ex\!\left[\twonorm{\wb_{1}^{(j)}}^{2}\right]-\twonorm{\zb}^{2}\Big)
  =\frac{n-1}{s^{(j)}}\twonorm{\zb}^{2}.
\]
The conditional form \eqref{eq:col-vector-var-cond} follows by applying the same calculation after
conditioning on $\Fcal_{j-1}$.
\end{proof}

\subsection{Convergence at the Improved Sample Size}
\label{app:column-ms}

\begin{proof}[\textbf{Proof of Theorem~\ref{thm:colev-ms}}]
\phantomsection\label{proof:column-ms}
The proof conditions on the past and uses Lemma~\ref{lem:colev-vector-var}. Let
\[
  \zb_{0}:=\Gb^{-1}\Sigmab^{-1}\Ub^{\ts}\bb,
  \qquad
  \zb_j:=\Rhat^{(j)}\zb_{j-1},\quad j=1,\dots,t.
\]
By \eqref{eq:telescope} and $\Vb^{\ts}\Vb=\Ib_{n}$,
\begin{equation}
  \twonorm{\xhatstar-\xstar}=\twonorm{\zb_{t}},
  \qquad
  \twonorm{\zb_{0}}=\twonorm{\xstar},
  \label{eq:z-error}
\end{equation}
where the second equality follows from
$\Vb\Gb^{-1}\Sigmab^{-1}\Ub^{\ts}\bb=\xstar$.

For an arbitrary $j\in\{1,\dots,t\}$, recall from \eqref{eq:R} that $\Rhat^{(j)}=-(\Gb+\Ehat^{(j)})^{-1}\Ehat^{(j)}.$
Since
\[
\Gb+\Ehat^{(j)}=\lambda\Sigmab^{-2}+\Vb^{\ts}\Sb^{(j)}\Sb^{(j)\ts}\Vb
  \succeq \lambda\Sigmab^{-2},
\]
we have $\twonorm{(\Gb+\Ehat^{(j)})^{-1}}\le\smax^{2}/\lambda$. Therefore,
conditionally on $\Fcal_{j-1}$,
\begin{align}
  \ex\!\left[\twonorm{\zb_j}^{2}\mid\Fcal_{j-1}\right]
  &\le \left(\frac{\smax^{2}}{\lambda}\right)^{2}
  \ex\!\left[\twonorm{\Ehat^{(j)}\zb_{j-1}}^{2}\mid\Fcal_{j-1}\right]\notag\\
  &=\left(\frac{\smax^{2}}{\lambda}\right)^{2}
    \frac{n-1}{s^{(j)}}\twonorm{\zb_{j-1}}^{2},
  \label{eq:colev-cond-var}
\end{align}
where the last equality uses Lemma~\ref{lem:colev-vector-var} with $\zb=\zb_{j-1}$. Iterating
\eqref{eq:colev-cond-var} by the tower property and then using \eqref{eq:z-error}
gives \eqref{eq:colev-ms}. The high-probability bound \eqref{eq:colev-hp}
follows from one terminal application of Markov's inequality.

For a common sketch size satisfying \eqref{eq:colev-var-scale}, the product
in \eqref{eq:colev-ms} is bounded by
\[
  \left[
  \left(\frac{\smax^{2}}{\lambda}\right)^{2}\frac{n-1}{s}
  \right]^t
  \le \varepsilon^{2t}.
\]
This proves the stated mean square contraction. The high-probability
statement follows by substituting the same sample-size bound into
\eqref{eq:colev-hp}.
\end{proof}
\subsection{Convergence on the Embedding Event}
\label{app:column-embedding}

The next result retains a uniform embedding event to avoid the factor $\smax^2/\lambda$. 
\begin{theorem}[Leverage score sampling convergence with an embedding event]
\label{thm:colev-struct}
Run Algorithm~\ref{alg:fresh} with LS. Fix $\varepsilon\in(0,1)$ and assume
$s^{(j)}\ge2n/\varepsilon^2$ for every $j$. Define
\begin{equation*}
  \delta_{\mathrm{ls}}(s)
  =4(n+1)\exp\!\left(-\frac{3\varepsilon^2s}{8n}\right),
  \qquad
  J_t=\bigcap_{j=1}^{t}
       \{\twonorm{\Ehat^{(j)}}\le\varepsilon\}.
\end{equation*}
Then
\begin{equation}
  \ex[\twonorm{\xhatstar-\xstar}^2\II_{J_t}]
  \le(1-\varepsilon)^{-2t}
       \left({\prod_{j=1}^{t}\frac{n-1}{s^{(j)}}}\right)\twonorm{\xstar}^2.
  \label{eq:colev-ms-struct}
\end{equation}
For any fixed $t$ and $\delta_{\mathrm{tail}}\in(0,1)$, with probability
at least $1-\delta_{\mathrm{tail}}-
\sum_{j=1}^{t}\delta_{\mathrm{ls}}(s^{(j)})$,
\begin{equation}
  \twonorm{\xhatstar-\xstar}
  \le\frac{(1-\varepsilon)^{-t}}{\sqrt{\delta_{\mathrm{tail}}}}
       \prod_{j=1}^{t}\sqrt{\frac{n-1}{s^{(j)}}}\twonorm{\xstar}.
  \label{eq:colev-hp-struct}
\end{equation}
In particular, fix $\delta_{\mathrm{emb}}\in(0,1)$. If the common sketch size satisfies
\begin{equation*}
  s\ge\max\left\{
    \frac{2n}{\varepsilon^2},\;
    \frac{8n}{3\varepsilon^2}
      \log\!\left(\frac{4t(n+1)}{\delta_{\mathrm{emb}}}\right)
  \right\},
\end{equation*}
then, with probability at least
$1-\delta_{\mathrm{tail}}-\delta_{\mathrm{emb}}$,
\begin{equation*}
  \twonorm{\xhatstar-\xstar}
  \le\frac{1}{\sqrt{\delta_{\mathrm{tail}}}}
    \left(\frac{\varepsilon}{1-\varepsilon}
      \sqrt{\frac{3(n-1)}
        {8n\log(4t(n+1)/\delta_{\mathrm{emb}})}}\right)^t
    \twonorm{\xstar}.
\end{equation*}
\end{theorem}

\begin{proof}
\phantomsection\label{proof:column-embedding}
If $\xstar=\zero$, then $\bb=\zero$ and every iterate is zero, so the conclusions hold immediately. Assume otherwise in the calculation below.
Let $B_j=\{\twonorm{\Ehat^{(j)}}\le\varepsilon\}$ and $J_j=\bigcap_{\ell=1}^{j}B_\ell$, with $J_0=\Omega$.
Use the same variables as in the proof of Theorem~\ref{thm:colev-ms}:
\[
  \zb_{0}:=\Gb^{-1}\Sigmab^{-1}\Ub^{\ts}\bb,
  \qquad
  \zb_j:=\Rhat^{(j)}\zb_{j-1},\quad j=1,\dots,t.
\]
By \eqref{eq:telescope}, $\Vb^{\ts}\Vb=\Ib_n$, and \eqref{eq:tail-is-xstar},
\[
  \twonorm{\xhatstar-\xstar}=\twonorm{\zb_t},
  \qquad
  \twonorm{\zb_0}=\twonorm{\xstar}.
\]
On $B_j$, we have $\Ib_n+\Ehat^{(j)}\succeq(1-\varepsilon)\Ib_n$, and hence
\[
  \Gb+\Ehat^{(j)}
  =\lambda\Sigmab^{-2}+\Ib_n+\Ehat^{(j)}
  \succeq(1-\varepsilon)\Ib_n .
\]
Thus $\twonorm{(\Gb+\Ehat^{(j)})^{-1}}\le(1-\varepsilon)^{-1}$. Since
$\Rhat^{(j)}=-(\Gb+\Ehat^{(j)})^{-1}\Ehat^{(j)}$ and
$J_j=J_{j-1}\cap B_j$,
\begin{equation*}
  \twonorm{\zb_j}^{2}\II_{J_j}
  \le
  (1-\varepsilon)^{-2}
  \twonorm{\Ehat^{(j)}\zb_{j-1}}^{2}\II_{J_{j-1}}.
\end{equation*}
Condition on $\Fcal_{j-1}$. The event $J_{j-1}$ and vector $\zb_{j-1}$ are
$\Fcal_{j-1}$-measurable, while $\sigma(\Sb^{(j)})$ is independent of $\Fcal_{j-1}$. Lemma~\ref{lem:colev-vector-var} with $\zb = \zb_{j-1}$
therefore gives

\begin{align}
     \ex\!\left[\twonorm{\zb_j}^{2}\II_{J_j}\right] &\leq (1-\varepsilon)^{-2}  \ex \! \left[ \twonorm{\Ehat^{(j)} \zb_{j-1} }^2 \II_{J_{j-1}}  \right]\notag\\
     &= (1-\varepsilon)^{-2}  \ex \! \left[   \ex \! \left[ \twonorm{\Ehat^{(j)} \zb_{j-1} }^2 \II_{J_{j-1}} \mid \Fcal_{j-1} \right]  \right]\notag\\
     &= (1-\varepsilon)^{-2}  \ex \! \left[    \II_{J_{j-1}}\ex \! \left[ \twonorm{\Ehat^{(j)} \zb_{j-1} }^2 \mid \Fcal_{j-1} \right]  \right]\notag\\
     &= (1-\varepsilon)^{-2} \frac{n-1}{s^{(j)}} \ex\!\left[\twonorm{\zb_{j-1}}^{2}\II_{J_{j-1}}\right].\label{eq:colev-struct-rec}
\end{align}

Iterating \eqref{eq:colev-struct-rec} proves \eqref{eq:colev-ms-struct}.

\begin{align*}
   &\pr{ \twonorm{\xhatstar-\xstar}  \ge \frac{1}{\sqrt{\delta_{\mathrm{tail}}}}\, (1-\varepsilon)^{-t}  \prod_{j=1}^{t}\sqrt{\frac{n-1}{s^{(j)}}}\,  \twonorm{\xstar}, \,J_t} \\
   &= \pr{\twonorm{\xhatstar-\xstar}^2  \ge \frac{1}{\delta_{\mathrm{tail}}} (1-\varepsilon)^{-2t} \prod_{j=1}^{t}\frac{n-1}{s^{(j)}}\, \twonorm{\xstar}^{2} ,\, J_t}\\
   &= \pr{\twonorm{\xhatstar-\xstar}^2 \II_{J_t}  \ge \frac{1}{\delta_{\mathrm{tail}}} (1-\varepsilon)^{-2t} \prod_{j=1}^{t}\frac{n-1}{s^{(j)}}\, \twonorm{\xstar}^{2} }\\
   &\leq \delta_{\mathrm{tail}},
\end{align*}

where the last inequality follows from Markov's inequality and \eqref{eq:colev-ms-struct}. Lemma~\ref{lem:cyd-struct-tail}, applied to $\Vb$, gives
$\pr{B_j^c}\le\delta_{\mathrm{ls}}(s^{(j)})$ for every $j$. Hence
$\pr{J_t^c}\le\sum_{j=1}^{t}\delta_{\mathrm{ls}}(s^{(j)})$ by a union bound.
Combining the two bounds proves \eqref{eq:colev-hp-struct}.

\emph{Common sketch size.}
The sample-size assumption gives
$\delta_{\mathrm{ls}}(s)\le\delta_{\mathrm{emb}}/t$, so the embedding failure probability in
Theorem~\ref{thm:colev-struct} is at most $\delta_{\mathrm{emb}}$. It also implies
\[
  \sqrt{\frac{n-1}{s}}
  \le
  \varepsilon
  \sqrt{\frac{3(n-1)}
             {8n\log(4t(n+1)/\delta_{\mathrm{emb}})}}.
\]
Substituting this bound into \eqref{eq:colev-hp-struct} proves the claim.
\end{proof}
\subsection{Squared Contraction of Algorithmic Bias}
\label{app:column-bias}

Here bias refers to the mean solver error over the sketches, $\ex[\xhatstar-\xstar]$, with $\Ab$ and $\bb$ held fixed. It is distinct from the statistical bias of ridge regression. Related works study inversion bias associated with sketching~\citep{derezinski2021sparse,niu2025fundamental}. 

\begin{theorem}[Squared contraction of algorithmic bias]
\label{thm:colev}
Fix $\varepsilon\in(0,1)$ and define
\begin{equation*}
  C= \frac{\smax^2}{\lambda} + \frac{\smax^2}{\smin^2},
  \qquad
  \delta(s)=4(n+1)\exp\!\left(-\frac{3\varepsilon^2s}{32n}\right).
\end{equation*}
If $s^{(j)}\ge8n/\varepsilon^2$ for every $j$, then
\begin{equation}
  \twonorm{\ex[\xhatstar-\xstar]}
  \le\prod_{j=1}^{t}\left(\frac{\varepsilon^2}{2}
           +C(n-1)^2\delta(s^{(j)})\right)\twonorm{\xstar}.
  \label{eq:colev-bias-bound}
\end{equation}
In particular, suppose all iterations use a common sketch size satisfying
\begin{equation*}
  s\ge\max\left\{
    \frac{8n}{\varepsilon^2},\;
    \frac{32n}{3\varepsilon^2}
    \log\!\frac{8C(n+1)(n-1)^2}{\varepsilon^2}
  \right\}.
\end{equation*}
Then
\begin{equation*}
  \twonorm{\ex[\xhatstar-\xstar]}
  \le\varepsilon^{2t}\twonorm{\xstar}.
\end{equation*}
\end{theorem}

\begin{proof}
\phantomsection\label{proof:column-bias}

\emph{Step 1 (factorize the expectation).} Take expectations in
\eqref{eq:telescope}. The factors $\Vb$ and $\Gb^{-1}\Sigmab^{-1}\Ub^{\ts}\bb$
are deterministic, and the $\Rhat^{(j)}$ are independent (each depends only on its own $\Sb^{(j)}$), so
\begin{equation*}
  \ex[\xhatstar-\xstar]
  =(-1)^{t-1}\Vb\,\ex[\Rhat^{(t)}]\cdots\ex[\Rhat^{(1)}]\,\Gb^{-1}\Sigmab^{-1}\Ub^{\ts}\bb.
\end{equation*}

\emph{Step 2 (an exact identity for $\ex[\Rhat^{(j)}]$).} Fix $j$ and write
$\Mb:=\Gb^{-1}\Ehat^{(j)}$, so that by definition \eqref{eq:R}
\begin{equation*}
  \Rhat^{(j)}=(\Ib_{n}+\Mb)^{-1}-\Ib_{n}.
\end{equation*}
We claim that
\begin{equation*}
  (\Ib_{n}+\Mb)^{-1}-\Ib_{n}=-\Mb+(\Ib_{n}+\Mb)^{-1}\Mb^{2}.
\end{equation*}
Indeed,
\begin{align*}
  (\Ib_{n}+\Mb)^{-1}-\Ib_{n}+\Mb
  &=(\Ib_{n}+\Mb)^{-1}\big[\Ib_{n}-(\Ib_{n}+\Mb)+(\Ib_{n}+\Mb)\Mb\big]\\
  &=(\Ib_{n}+\Mb)^{-1}\big[\Ib_{n}-\Ib_{n}-\Mb+\Mb+\Mb^{2}\big]
   =(\Ib_{n}+\Mb)^{-1}\Mb^{2},
\end{align*}
Therefore
\begin{equation*}
  \Rhat^{(j)}=-\Mb+(\Ib_{n}+\Mb)^{-1}\Mb^{2},
\end{equation*}
and since $\ex[\Mb]=\zero$,
\begin{equation}
  \ex[\Rhat^{(j)}]=-\ex[\Mb]+\ex\!\big[(\Ib_{n}+\Mb)^{-1}\Mb^{2}\big]
              =\ex\!\big[(\Ib_{n}+\Mb)^{-1}\Mb^{2}\big].
  \label{eq:exactER}
\end{equation}
This is exact: it uses only $\ex[\Mb]=\zero$ and the invertibility of
$\Ib_{n}+\Mb$ (guaranteed by $\lambda>0$), with no smallness assumption on $\Mb$. The linear term has cancelled, leaving a quantity \emph{quadratic} in $\Mb$.

\emph{Step 3 (bound the quadratic).} A deterministic bound on
$\twonorm{(\Ib_{n}+\Mb)^{-1}}$ holds always: $\Ib_{n}+\Mb=\Gb^{-1}(\Gb+\Ehat^{(j)})$
with $\Gb+\Ehat^{(j)}=\lambda\Sigmab^{-2}+\Vb^{\ts}\Sb^{(j)}\Sb^{(j)\ts}\Vb\succeq\lambda\Sigmab^{-2}\succ0$,
so
\begin{equation}
  \twonorm{(\Ib_{n}+\Mb)^{-1}}
  =\twonorm{(\Gb+\Ehat^{(j)})^{-1}\Gb}
  \le\twonorm{(\Gb+\Ehat^{(j)})^{-1}}\,\twonorm{\Gb}
  \le\frac{\smax^{2}}{\lambda}\Big(1+\frac{\lambda}{\smin^{2}}\Big)=C.
  \label{eq:detinv}
\end{equation}
Split on the event $B_j=\{\twonorm{\Ehat^{(j)}}\le\varepsilon/2\}$, which by
Lemma~\ref{lem:cyd-struct-tail} satisfies $\pr{B_j^{c}}\le\delta(s^{(j)})$. On $B_j$,
$\twonorm{\Mb}\le\twonorm{\Ehat^{(j)}}\le\varepsilon/2$ because
$\twonorm{\Gb^{-1}}\le1$.
From \eqref{eq:exactER},
\begin{align*}
  \twonorm{\ex[\Rhat^{(j)}]} &\leq \ex\! \big[\twonorm{(\Ib_{n}+\Mb)^{-1} {\Mb}^{2}} \big]\\
  &\le\ex\!\big[\twonorm{(\Ib_{n}+\Mb)^{-1}}\,\twonorm{\Mb}^{2}\big]\\
&=\ex\!\big[\twonorm{(\Ib_{n}+\Mb)^{-1}}\twonorm{\Mb}^{2}\,\II_{B_j}\big]+\ex\!\big[\twonorm{(\Ib_{n}+\Mb)^{-1}}\twonorm{\Mb}^{2}\,\II_{B_j^{c}}\big],
\end{align*}
where in the first inequality we used Jensen's inequality and for the second we used submultiplicativity of the spectral norm. On $B_j$ we have $\twonorm{(\Ib_{n}+\Mb)^{-1}}\le1/(1-\varepsilon/2)$ (since
$\twonorm{\Mb}\le\varepsilon/2<1$, so $\twonorm{(\Ib_{n}+\Mb)^{-1}} = \twonorm{\sum_{\ell=0}^\infty (-1)^\ell \Mb^\ell} \le \sum_{\ell=0}^\infty \twonorm{\Mb}^\ell\le 1/(1-\twonorm{\Mb})$)
and $\twonorm{\Mb}^{2}\le(\varepsilon/2)^{2}$, so the
first term is at most $\frac{(\varepsilon/2)^{2}}{1-\varepsilon/2}\le\varepsilon^{2}/2$
for $\varepsilon\le1$. On $B_j^{c}$ we use the deterministic bound \eqref{eq:detinv},
giving the second term $\le C\,\ex[\twonorm{\Mb}^{2}\II_{B_j^{c}}]$.

It remains to bound $\ex[\twonorm{\Mb}^{2}\II_{B_j^{c}}]$. Leverage score sampling yields
\begin{equation*}
  \zero\preceq\Vb^{\ts}\Sb^{(j)}\Sb^{(j)\ts}\Vb\preceq n\Ib_{n},
\end{equation*}
since $\Vb^{\ts}\Sb^{(j)}\Sb^{(j)\ts}\Vb
=\tfrac{1}{s^{(j)}}\sum_{k=1}^{s^{(j)}}p_{i_{k}}^{-1}(\Vb_{i_{k}\ast})^{\ts}(\Vb_{i_{k}\ast})$ and each
rank-one term obeys
$p_{i_{k}}^{-1}(\Vb_{i_{k}\ast})^{\ts}(\Vb_{i_{k}\ast})
\preceq(\twonorm{\Vb_{i_{k}\ast}}^{2}/p_{i_{k}})\Ib_{n}=n\Ib_{n}$, using
$p_{i}=\twonorm{\Vb_{i\ast}}^{2}/n$. Hence
$-\Ib_{n}\preceq\Ehat^{(j)}\preceq(n-1)\Ib_{n}$, and since $\twonorm{\Gb^{-1}}\le1$,
\begin{equation*}
  \twonorm{\Mb}\le\twonorm{\Gb^{-1}}\,\twonorm{\Ehat^{(j)}}\le n-1.
\end{equation*}
With $\pr{B_j^{c}}\le\delta(s^{(j)})$ this gives the deterministic control of the
remainder,
\begin{equation*}
  \ex\!\big[\twonorm{\Mb}^{2}\II_{B_j^{c}}\big]
  \le (n-1)^{2}\,\pr{B_j^{c}}\le (n-1)^{2}\,\delta(s^{(j)}).
\end{equation*}
Altogether,
\begin{equation*}
  \twonorm{\ex[\Rhat^{(j)}]}
  \le\tfrac{\varepsilon^{2}}{2}+C(n-1)^{2}\,\delta(s^{(j)}).
\end{equation*}

\emph{Step 4 (conclude).} By Step~1, submultiplicativity of the spectral norm,
$\twonorm{\Vb}=1$, and \eqref{eq:tail-is-xstar},
\begin{equation*}
  \twonorm{\ex[\xhatstar-\xstar]}
  \le\prod_{j=1}^{t}\twonorm{\ex[\Rhat^{(j)}]}\,\twonorm{\xstar}
  \le\prod_{j=1}^{t}\Big(\tfrac{\varepsilon^{2}}{2}+C(n-1)^{2}\,\delta(s^{(j)})\Big)\twonorm{\xstar}.
\end{equation*}

\emph{Common sketch size.}
The stated $s$ gives
$\delta(s)=4(n+1)\exp\!\big(-\tfrac{3\varepsilon^{2}s}{32n}\big)
\le\varepsilon^{2}/(2C(n-1)^{2})$, so
$C(n-1)^{2}\,\delta(s)\le\varepsilon^{2}/2$ and every factor in
\eqref{eq:colev-bias-bound} is at most $\varepsilon^{2}$.
\end{proof}

\section{Ridge Leverage Analysis}
\label{app:ridge}

Throughout this appendix, $\Xb=\Vb\Slam$ and
$w_i=\twonorm{\Xb_{i\ast}}^2$, so $\sum_iw_i=\dlam>0$. We write
$c_\lambda=(\smax^2+\lambda)/\lambda$ as in~\eqref{eq:clambda}.

\subsection{Error identity and conditional variance}
\label{sec:ridge-tools}

We also use the inverse error matrix
\begin{equation*}
 \Rb^{(j)}=(\Ib_n+\Eb^{(j)})^{-1}-\Ib_n
          =-(\Ib_n+\Eb^{(j)})^{-1}\Eb^{(j)}.
\end{equation*}

\begin{proof}[\textbf{Proof of Lemma~\ref{lem:periter}}]
\phantomsection\label{proof:ridge-identity}
From \eqref{eq:clambda},
\begin{equation*}
  \Ib_{n}+\Eb^{(j)}=(\Ib_{n}-\Slam^{2})+\Xb^{\ts}\Sb^{(j)}\Sb^{(j)\ts}\Xb.
\end{equation*}
The second matrix is a Gram matrix, hence symmetric positive semidefinite. The first is
diagonal with entries $1-\sigma_{i}^{2}/(\sigma_{i}^{2}+\lambda)=\lambda/(\sigma_{i}^{2}+\lambda)>0$,
so $\Ib_{n}-\Slam^{2}\succ0$ with smallest eigenvalue
$\min_{i}\lambda/(\sigma_{i}^{2}+\lambda)=\lambda/(\smax^{2}+\lambda)$.
Adding a PSD matrix only raises eigenvalues, so $\Ib_{n}+\Eb^{(j)}$ is symmetric
positive definite with
$\lambda_{\min}(\Ib_{n}+\Eb^{(j)})\ge\lambda/(\smax^{2}+\lambda)$. For a symmetric
positive-definite matrix $\Ab, \twonorm{\Ab^{-1}}=1/\lambda_{\min}(\Ab)$, gives
\begin{equation}
    \twonorm{(\Ib_n+\Eb^{(j)})^{-1}}\le\frac{\smax^2+\lambda}{\lambda}=c_\lambda. \label{eq:clambda_bound}
\end{equation}

By \citet[Lemma 5, eqn.~(51)]{chowdhury2018iterative}\footnote{Even though the proof steps in~\citet{chowdhury2018iterative} use the series form, the same result holds with our definition of $\Rb^{(j)}$.}
\begin{equation*}
  \widetilde{\xb}^{(j)}=\xb^{\star(j)}+\Vb\Slam\,\Rb^{(j)}\,\Slam\Sigmab^{-1}\Ub^{\ts}\bb^{(j)},
\end{equation*}
so with $\Rb^{(j)}=-(\Ib_{n}+\Eb^{(j)})^{-1}\Eb^{(j)}$,
\begin{equation}
  \xb^{\star(j)}-\widetilde{\xb}^{(j)}
  =\Vb\Slam(\Ib_{n}+\Eb^{(j)})^{-1}\,\Eb^{(j)}\Slam\Sigmab^{-1}\Ub^{\ts}\bb^{(j)}.
  \label{eq:periter-raw}
\end{equation}
We simplify the tail. Using $\Eb^{(j)}=\Slam\Vb^{\ts}\Sb^{(j)}\Sb^{(j)\ts}\Vb\Slam-\Slam^{2}$,
$\Vb^{\ts}\Vb=\Ib_{n}$, $\Xb=\Vb\Slam$, and Lemma~\ref{lem:chow-acc}
($\xb^{\star(j)}=\Vb\Slam^{2}\Sigmab^{-1}\Ub^{\ts}\bb^{(j)}$),
\begin{align}
  \Eb^{(j)}\Slam\Sigmab^{-1}\Ub^{\ts}\bb^{(j)}
  &=\Slam\Vb^{\ts}\Sb^{(j)}\Sb^{(j)\ts}\Vb\Slam^{2}\Sigmab^{-1}\Ub^{\ts}\bb^{(j)}
   -\Slam\Vb^{\ts}\Vb\Slam^{2}\Sigmab^{-1}\Ub^{\ts}\bb^{(j)}\notag\\
  &=\Slam\Vb^{\ts}\Sb^{(j)}\Sb^{(j)\ts}\,\xb^{\star(j)}-\Slam\Vb^{\ts}\,\xb^{\star(j)}
   =\Xb^{\ts}\big(\Sb^{(j)}\Sb^{(j)\ts}-\Ib_{d}\big)\xb^{\star(j)},
  \label{eq:tail-simpl}
\end{align}
where the last line uses $\Slam\Vb^{\ts}=\Xb^{\ts}$ and
$\Vb\Slam^{2}\Sigmab^{-1}\Ub^{\ts}\bb^{(j)}=\xb^{\star(j)}$. Substituting
\eqref{eq:tail-simpl} into \eqref{eq:periter-raw} gives the stated identity.
Finally $\twonorm{\Vb\Slam}\le\twonorm{\Vb}\twonorm{\Slam}\le1$ and~\eqref{eq:clambda_bound} give $\twonorm{\xb^{\star(j)}-\widetilde{\xb}^{(j)}} \leq c_\lambda\twonorm{\Xb^{\ts}(\Sb^{(j)}\Sb^{(j)\ts}-\Ib_d)
\xb^{\star(j)}}$. If
$\twonorm{\Eb^{(j)}}\le\varepsilon<1$, then the Neumann series gives
$\twonorm{(\Ib_{n}+\Eb^{(j)})^{-1}}\le 1/(1-\varepsilon)$, proving
\eqref{eq:periter-struct}.
\end{proof}

\begin{lemma}[Variance of a sketched matrix product]
\label{lem:variance}
For fixed $\xb\in\R^d$ and draw $\Sb$ by Algorithm~\ref{alg:sample} with size $s$ and probabilities $p$. Then
\begin{align*}
 \ex\twonorm{\Xb^{\ts}(\Sb\Sb^{\ts}-\Ib_d)\xb}^2 &\le\frac1s\sum_i\frac{|\xb_i|^2w_i}{p_i}.
\end{align*}
If $\xb$ and $p$ are measurable with respect to a sigma algebra $\Gcal$, the same calculation gives
\begin{equation}
 \ex\!\left[\twonorm{\Xb^{\ts}(\Sb\Sb^{\ts}-\Ib_d)\xb}^2
           \mid\Gcal\right]
 \le\frac1s\sum_i\frac{|\xb_i|^2w_i}{p_i}.
 \label{eq:variance-cond}
\end{equation}
\end{lemma}

\begin{proof}\label{proof:ridge-variance}
Use Theorem 22 from~\citet{drineas2017lectures}. The conditional form \eqref{eq:variance-cond} follows by applying the same calculation after conditioning on $\Gcal$.
\end{proof}

\begin{proof}[\textbf{Proof of Lemma~\ref{lem:ridge-vector-var}}]
\phantomsection\label{proof:ridge-vector-var}
Condition on $\Fcal_{j-1}$, so that $\xb$ is fixed. Apply the exact variance
identity in Lemma~\ref{lem:variance} with $p_i=w_i/\dlam$ to obtain
\[
 \ex\!\left[
 \twonorm{\Xb^{\ts}(\Sb^{(j)}\Sb^{(j)\ts}-\Ib_d)\xb}^2
 \mid\Fcal_{j-1}\right]
 \le \frac1{s^{(j)}}\left(
 \sum_i\frac{|\xb_i|^2w_i}{w_i/\dlam}
 \right) = \frac{\dlam}{s^{(j)}}\twonorm{\xb}^2.
\]

\end{proof}

\subsection{A convergence theorem for residual-aware distributions}
\label{app:adaptive-master}

\begin{theorem}[Convergence from conditional variance]
\label{thm:master-var}
For the iteration in \eqref{eq:iteration}, let $(\Hcal_j)_{j=0}^t$ be a filtration
such that $\Hcal_{j-1}$ contains all randomness available before drawing
$\Sb^{(j)}$, for $j=1,\ldots,t$. Suppose $p^{(j)}$ and
$\xb^{\star(j)}$ are $\Hcal_{j-1}$-measurable and, conditionally on $\Hcal_{j-1}$, the sketch contains $s^{(j)}$ independent draws from $p^{(j)}$.
For nonzero residual solutions, define
\begin{equation}
 R_p^{(j)}=
 \frac{\sum_i|\xb_i^{\star(j)}|^2w_i/p_i^{(j)}}
      {\dlam\twonorm{\xb^{\star(j)}}^2},
 \qquad P_t=\prod_{j=1}^{t}R_p^{(j)}.
 \label{eq:master-factor}
\end{equation}
Set $R_p^{(j)}=1$ when $\xb^{\star(j)}=\zero$. Then $P_t>0$ and
\begin{equation}
 \ex\!\left[\frac{\twonorm{\xhatstar-\xstar}^2}{P_t}\right]
 \le c_\lambda^{2t}\prod_{j=1}^{t}\frac{\dlam}{s^{(j)}}
                      \twonorm{\xstar}^2.
 \label{eq:master-ms}
\end{equation}
For every $\delta\in(0,1)$, with probability at least $1-\delta$,
\begin{equation}
 \twonorm{\xhatstar-\xstar}
 \le\delta^{-1/2}c_\lambda^t
       \prod_{j=1}^{t}\sqrt{\frac{R_p^{(j)}\dlam}{s^{(j)}}}
       \twonorm{\xstar}.
 \label{eq:master-hp}
\end{equation}
\end{theorem}

\begin{proof}
\phantomsection\label{proof:adaptive-master}
Set $\eb_j:=\xb^{\star(j)}-\widetilde{\xb}^{(j)}$ and $\eb_0:=\xstar$, and define
$P_j:=\prod_{\ell=1}^{j}R_p^{(\ell)}$ with $P_0:=1$.

Let
\[
  \yb_j:=\Xb^{\ts}(\Sb^{(j)}\Sb^{(j)\ts}-\Ib_d)\xb^{\star(j)} .
\]
Lemma~\ref{lem:periter} gives
$\twonorm{\eb_j}\le c_\lambda\twonorm{\yb_j}$ for every sketch realization.
Since $p^{(j)}$ and $\xb^{\star(j)}$ are $\Hcal_{j-1}$-measurable, \eqref{eq:variance-cond} and the definition of $R_p^{(j)}$ give
\begin{equation}
  \ex\!\left[\twonorm{\eb_j}^{2}\mid\Hcal_{j-1}\right]
  \le
  c_\lambda^{2}\frac{R_p^{(j)}\dlam}{s^{(j)}}
  \twonorm{\xb^{\star(j)}}^{2}.
  \label{eq:master-cond-rec-raw}
\end{equation}
By Lemma~\ref{lem:chow-rec}, $\twonorm{\xb^{\star(j)}}=\twonorm{\eb_{j-1}}$.
Since $P_j=P_{j-1}R_p^{(j)}$ is $\Hcal_{j-1}$-measurable,
\begin{equation}
  \ex\!\left[
  \frac{\twonorm{\eb_j}^{2}}{P_j}\,\middle|\,\Hcal_{j-1}
  \right]
  \le
  c_\lambda^{2}\frac{\dlam}{s^{(j)}}
  \frac{\twonorm{\eb_{j-1}}^{2}}{P_{j-1}}.
  \label{eq:master-selfnorm-rec}
\end{equation}
Iterating \eqref{eq:master-selfnorm-rec} with the tower property and using
Lemma~\ref{lem:chow-decomp}
($\twonorm{\xhatstar-\xstar}=\twonorm{\eb_t}$) gives \eqref{eq:master-ms}.
Markov's inequality applied to $\twonorm{\eb_t}^{2}/P_t$ gives
\eqref{eq:master-hp}.
\end{proof}

\begin{proof}[\textbf{Proof of Theorem~\ref{thm:ridgelev}}]
\phantomsection\label{proof:ridge-main}
Apply Theorem~\ref{thm:master-var} with $\Hcal_j=\Fcal_j$ and
$p_i^{(j)}=w_i/\dlam$. The residual solution $\xb^{\star(j)}$ is
$\Fcal_{j-1}$-measurable, and each fresh sketch consists of independent
draws from this fixed distribution. For every nonzero residual solution,
\[
 R_p^{(j)}
 =\frac{\sum_i|\xb_i^{\star(j)}|^2w_i/(w_i/\dlam)}
       {\dlam\twonorm{\xb^{\star(j)}}^2}
 =1.
\]
For zero residual solutions, $R_p^{(j)}=1$ by definition. Thus $P_t=1$,
and \eqref{eq:master-ms} and \eqref{eq:master-hp} give
\eqref{eq:ridge-ms-det} and \eqref{eq:ridge-hp-det}, respectively.

For a common sketch size satisfying \eqref{eq:ridge-var-sample},
$c_\lambda^2\dlam/s\le\varepsilon^2$; substituting this into the two bounds
gives the final claims.
\end{proof}

\subsection{Convergence on the embedding events}
\label{app:ridge-struct}

\begin{theorem}[Ridge leverage score sampling convergence with an embedding event]
\label{thm:ridgelev-struct}
Run Algorithm~\ref{alg:fresh} with RLS.
Fix $\varepsilon\in(0,1)$ and assume
$s^{(j)}\ge2\dlam/\varepsilon^2$ for every $j$. Define
\begin{equation*}
 \delta_{\mathrm{rls}}(s)
 =4(1+\dlam)\exp\!\left(-\frac{3\varepsilon^2s}{8\dlam}\right),
 \qquad
 H_j=\bigcap_{\ell=1}^{j}\{\twonorm{\Eb^{(\ell)}}\le\varepsilon\}.
\end{equation*}
Then
\begin{equation}
 \ex[\twonorm{\xhatstar-\xstar}^2\II_{H_t}]
 \le(1-\varepsilon)^{-2t}
       \prod_{j=1}^{t}\frac{\dlam}{s^{(j)}}\twonorm{\xstar}^2.
 \label{eq:ridge-ms-struct}
\end{equation}
For $\delta_{\mathrm{tail}}\in(0,1)$, with probability at least
$1-\delta_{\mathrm{tail}}-\sum_{j=1}^{t}\delta_{\mathrm{rls}}(s^{(j)})$,
\begin{equation}
 \twonorm{\xhatstar-\xstar}
 \le\delta_{\mathrm{tail}}^{-1/2}(1-\varepsilon)^{-t}
       \prod_{j=1}^{t}\sqrt{\frac{\dlam}{s^{(j)}}}\twonorm{\xstar}.
 \label{eq:ridge-hp-struct}
\end{equation}
In particular, fix $\delta_{\mathrm{emb}}\in(0,1)$.
If every iteration uses a common size
\begin{equation*}
 s\ge\max\left\{\frac{2\dlam}{\varepsilon^2},
 \frac{8\dlam}{3\varepsilon^2}
 \log\!\left(\frac{4t(1+\dlam)}{\delta_{\mathrm{emb}}}\right)\right\},
\end{equation*}
then, with probability at least
$1-\delta_{\mathrm{tail}}-\delta_{\mathrm{emb}}$,
\begin{equation*}
 \twonorm{\xhatstar-\xstar}
 \le\delta_{\mathrm{tail}}^{-1/2}
 \left(\frac{\varepsilon}{1-\varepsilon}
 \sqrt{\frac{3}{8\log(4t(1+\dlam)/\delta_{\mathrm{emb}})}}\right)^t
 \twonorm{\xstar}.
\end{equation*}
\end{theorem}

\begin{proof}
\phantomsection\label{proof:ridge-embedding}
Let $G_j=\{\twonorm{\Eb^{(j)}}\le\varepsilon\}$, $H_j=\bigcap_{\ell=1}^{j}G_\ell$, and $H_0=\Omega$.
Set $\eb_j:=\xb^{\star(j)}-\widetilde{\xb}^{(j)}$ and $\eb_0:=\xstar$. Also set
\[
  \yb_j:=\Xb^{\ts}(\Sb^{(j)}\Sb^{(j)\ts}-\Ib_{d})\xb^{\star(j)}.
\]
On $G_j$,~\eqref{eq:periter-struct} gives
$\twonorm{\eb_j}\le(1-\varepsilon)^{-1}\twonorm{\yb_j}$. Since
$H_j=H_{j-1}\cap G_j$,
\begin{equation*}
  \twonorm{\eb_j}^{2}\II_{H_j}
  \le
  (1-\varepsilon)^{-2}\twonorm{\yb_j}^{2}\II_{H_{j-1}}\II_{G_j}
  \le
  (1-\varepsilon)^{-2}\twonorm{\yb_j}^{2}\II_{H_{j-1}}.
\end{equation*}
The last inequality is the key point: we use $G_j$ to bound the inverse, then
drop $\II_{G_j}$ before conditioning, so the conditional sampling law is not
changed.

Now take expectations and condition on $\Fcal_{j-1}$. The event $H_{j-1}$ and
the vector $\xb^{\star(j)}$ are $\Fcal_{j-1}$-measurable, while $\Sb^{(j)}$ is
fresh. Substituting $p_i=w_i/\dlam$ into \eqref{eq:variance-cond} gives
$\sum_i\abs{\xb_i^{\star(j)}}^2w_i/p_i=\dlam\twonorm{\xb^{\star(j)}}^2$, and hence
\begin{align*}
  \ex\!\left[\twonorm{\eb_j}^{2}\II_{H_j}\right]
  &\le
  (1-\varepsilon)^{-2}
  \ex\!\left[
  \II_{H_{j-1}}\,
  \ex\!\left[\twonorm{\yb_j}^{2}\mid\Fcal_{j-1}\right]
  \right]\\
  &\le
  (1-\varepsilon)^{-2}
  \frac{\dlam}{s^{(j)}}
  \ex\!\left[
  \twonorm{\xb^{\star(j)}}^{2}\II_{H_{j-1}}
  \right].
\end{align*}
By Lemma~\ref{lem:chow-rec}, $\twonorm{\xb^{\star(j)}}=\twonorm{\eb_{j-1}}$ for
each $j$, with $\eb_0=\xstar$. Therefore
\begin{equation}
  \ex\!\left[\twonorm{\eb_j}^{2}\II_{H_j}\right]
  \le
  (1-\varepsilon)^{-2}
  \frac{\dlam}{s^{(j)}}
  \ex\!\left[\twonorm{\eb_{j-1}}^{2}\II_{H_{j-1}}\right].
  \label{eq:ridge-trunc-rec}
\end{equation}
Iterating \eqref{eq:ridge-trunc-rec} and using Lemma~\ref{lem:chow-decomp}
($\twonorm{\xhatstar-\xstar}=\twonorm{\eb_t}$) gives \eqref{eq:ridge-ms-struct}.

For the high-probability statement, Markov's inequality applied to
$\twonorm{\eb_t}^{2}\II_{H_t}$ gives
\[
  \pr{\twonorm{\eb_t}>
  \frac{1}{\sqrt{\delta_{\mathrm{tail}}}}\,
  (1-\varepsilon)^{-t}
  \prod_{j=1}^{t}\sqrt{\frac{\dlam}{s^{(j)}}}\twonorm{\xstar},
  \ H_t}
  \le\delta_{\mathrm{tail}}.
\]
Lemma~\ref{lem:cyd-struct-tail}, applied to $\Xb$, gives
$\pr{G_j^c}\le\delta_{\mathrm{rls}}(s^{(j)})$ for every $j$. Hence, by a union
bound, $\pr{H_t^c}\le\sum_{j=1}^{t}\delta_{\mathrm{rls}}(s^{(j)})$. Combining
the two bounds proves \eqref{eq:ridge-hp-struct}.

\emph{Common sketch size.}
The sample size assumption gives $\delta_{\mathrm{rls}}(s)\le\delta_{\mathrm{emb}}/t$,
so the embedding failure probability in Theorem~\ref{thm:ridgelev-struct} is at most
$\delta_{\mathrm{emb}}$. It also implies
\[
  \sqrt{\frac{\dlam}{s}}
  \le
  \varepsilon\sqrt{\frac{3}{8\log(4t(1+\dlam)/\delta_{\mathrm{emb}})}}.
\]
Substituting this bound into \eqref{eq:ridge-hp-struct} proves the claim.
\end{proof}

\section{Oracle and Adaptive Analysis}
\label{app:oracle}

Throughout this appendix, $c_\lambda=(\smax^2+\lambda)/\lambda$ as in~\eqref{eq:clambda}.

\subsection{Oracle convergence}
\label{sec:oracle-thy}

\begin{corollary}[Oracle convergence]
\label{cor:oracle}\label{cor:oracle-var-scale}
Run Algorithm~\ref{alg:fresh} with Oracle. Set $\rho^{(j)}=\rho(\xb^{\star(j)})$ and
$\Theta_t=\prod_{j=1}^{t}\rho^{(j)}$.
Then
\begin{equation}
 \ex\!\left[\frac{\twonorm{\xhatstar-\xstar}^2}{\Theta_t}\right]
 \le c_\lambda^{2t}\prod_{j=1}^{t}\frac{\dlam}{s^{(j)}}
        \twonorm{\xstar}^2.
 \label{eq:oracle-selfnorm-ms}
\end{equation}
For every $\delta\in(0,1)$, with probability at least $1-\delta$,
\begin{equation}
 \twonorm{\xhatstar-\xstar}
 \le\delta^{-1/2}c_\lambda^t
       \prod_{j=1}^{t}\sqrt{\frac{\rho^{(j)}\dlam}{s^{(j)}}}
       \twonorm{\xstar}.
 \label{eq:oracle-selfnorm-hp}
\end{equation}
For a common size $s$, this is \eqref{eq:oracle-var-scale}.
If a deterministic $\bar\rho\in(0,1]$ bounds $\rho^{(j)}$ 
at every iteration with nonzero residual, and
$s\ge c_\lambda^2\bar\rho\dlam/\varepsilon^2$ for
$\varepsilon\in(0,1)$, then
\begin{equation}
 \ex\twonorm{\xhatstar-\xstar}^2
 \le\varepsilon^{2t}\twonorm{\xstar}^2,
 \qquad
 \twonorm{\xhatstar-\xstar}
 \le\delta^{-1/2}\varepsilon^t\twonorm{\xstar}
 \label{eq:oracle-eps-contract}
\end{equation}
where the second inequality holds with probability at least $1-\delta$.
\end{corollary}

\begin{proof}
\phantomsection\label{proof:oracle-convergence}
Apply Theorem~\ref{thm:master-var} with $\Hcal_j=\Fcal_j$ and
$p^{(j)}=p^\star(\xb^{\star(j)})$. By Proposition~\ref{prop:oracle} and
\eqref{eq:rho}, the master factor \eqref{eq:master-factor} is
$R_p^{(j)}=\rho^{(j)}$. Hence $P_t=\Theta_t$, and
\eqref{eq:master-ms}--\eqref{eq:master-hp} give
\eqref{eq:oracle-selfnorm-ms}--\eqref{eq:oracle-selfnorm-hp}.

For common $s$, the product in \eqref{eq:oracle-selfnorm-hp} satisfies
\[
  \prod_{j=1}^{t}\sqrt{\frac{\rho^{(j)}\dlam}{s}}
  =
  \left(\sqrt{\frac{\bar{\rho}_{t}\dlam}{s}}\right)^{t}.
\]
This proves \eqref{eq:oracle-var-scale}. If $\rho^{(j)}\le\bar{\rho}$ for all
$j$, then $\bar{\rho}_{t}\le\bar{\rho}$; substituting
$s\ge c_\lambda^{2}\bar{\rho}\dlam/\varepsilon^{2}$ gives
\eqref{eq:oracle-eps-contract}.

\end{proof}

\subsection{Adaptive proposals and a defensive mixture}
\label{sec:adaptive}

The pilot sketch $\Sb^{(0)}$ is part of the history:
\begin{equation*}
 \Fcal_0^{\mathrm{ad}}=\sigma(\Sb^{(0)}),
 \qquad
 \Fcal_j^{\mathrm{ad}}
 =\sigma(\Sb^{(0)},\Sb^{(1)},\ldots,\Sb^{(j)}).
\end{equation*}
The proposal in \eqref{eq:phat} is therefore fixed conditionally on
$\Fcal_{j-1}^{\mathrm{ad}}$.

\begin{corollary}[Pure adaptive proposals]
\label{cor:pure-adaptive-var}
Run Algorithm~\ref{alg:adaptive} with the pure proposal
$\widehat p^{(j)}$ and sketch sizes $s^{(j)}$.
At a nonzero residual solution, define
\begin{equation}
 R_{\mathrm{pa}}^{(j)}=
 \frac{\sum_i|\xb_i^{\star(j)}|^2w_i/\widehat p_i^{(j)}}
      {\dlam\twonorm{\xb^{\star(j)}}^2},
 \qquad P_{\mathrm{pa},t}=\prod_{j=1}^{t}R_{\mathrm{pa}}^{(j)}.
 \label{eq:pure-adaptive-factor}
\end{equation}
Then
\begin{equation*}
 \ex\!\left[\frac{\twonorm{\xhatstar-\xstar}^2}{P_{\mathrm{pa},t}}\right]
 \le c_\lambda^{2t}\prod_{j=1}^{t}\frac{\dlam}{s^{(j)}}
       \twonorm{\xstar}^2.
\end{equation*}
For every $\delta\in(0,1)$, with probability at least $1-\delta$,
\begin{equation*}
 \twonorm{\xhatstar-\xstar}
 \le\delta^{-1/2}c_\lambda^t
       \prod_{j=1}^{t}\sqrt{\frac{R_{\mathrm{pa}}^{(j)}\dlam}{s^{(j)}}}
       \twonorm{\xstar}.
\end{equation*}
\end{corollary}

\begin{proof}
\phantomsection\label{proof:adaptive-pure}

Apply Theorem~\ref{thm:master-var} with $\Hcal_j=\Fcal_j^{\mathrm{ad}}$ and
$p^{(j)}=\widehat p^{(j)}$. Then $R_p^{(j)}=R_{\mathrm{pa}}^{(j)}$ by
\eqref{eq:pure-adaptive-factor}.

\end{proof}

The corollary does not bound proposal quality: small probabilities can make
$R_{\mathrm{pa}}^{(j)}$ arbitrarily large. The mixture in \eqref{eq:mix}
bounds this factor uniformly.

\begin{corollary}[Defensive mixture]
\label{cor:mixture-var}
Use \eqref{eq:mix} with any measurable probability proposal
$\widehat p^{(j)}$ and a fixed $\alpha\in(0,1)$.
At a nonzero residual solution, its factor $R_{\mathrm{mix}}^{(j)}$ in
\eqref{eq:master-factor} satisfies
\begin{equation*}
 R_{\mathrm{mix}}^{(j)}\le\frac1\alpha,
 \qquad
 R_{\mathrm{mix}}^{(j)}\le
       \frac{R_{\mathrm{pa}}^{(j)}}{1-\alpha}.
\end{equation*}
Consequently,
\begin{equation}
 \ex\twonorm{\xhatstar-\xstar}^2
 \le c_\lambda^{2t}\prod_{j=1}^{t}\frac{\dlam}{\alpha s^{(j)}}
                         \twonorm{\xstar}^2.
 \label{eq:mixture-ms}
\end{equation}
For $\varepsilon,\delta\in(0,1)$ and a common size
$s\ge c_\lambda^2\dlam/(\alpha\varepsilon^2)$, the right side is at most
$\varepsilon^{2t}\twonorm{\xstar}^2$; with probability at least
$1-\delta$, the relative error is at most
$\delta^{-1/2}\varepsilon^t$.
\end{corollary}

\begin{proof}
\phantomsection\label{proof:adaptive-mixture}
Since $\widetilde p_i^{(j)}\ge\alpha w_i/\dlam$,
\[
 \sum_i\frac{|\xb_i^{\star(j)}|^2w_i}{\widetilde p_i^{(j)}}
 \le\frac{\dlam}{\alpha}\twonorm{\xb^{\star(j)}}^2.
\]
The second comparison follows from
$\widetilde p_i^{(j)}\ge(1-\alpha)\widehat p_i^{(j)}$.
The first comparison inserted into \eqref{eq:master-cond-rec-raw} and
iterated proves \eqref{eq:mixture-ms}; it also holds after termination.
The remaining statements follow by substitution and Markov's inequality.
\end{proof}

\section{Details of experiments and additional results}
\label{app:experiment-details}

\subsection{Implementation, metrics, and sampling budgets}
\label{app:experiment-implementation}
\label{app:implementation}
The experiments use the residual correction in~\eqref{eq:iteration} for
uniform, leverage, ridge leverage, oracle, adaptive, and mixture
sampling. Sampling is with replacement, with the rescaling defined in
Section~\ref{sec:setup}. The mixture places weight $\alpha=0.1$ on ridge
leverage probabilities. Each trial runs for a fixed number $t$ of updates;
the zero initial iterate is included in error curves. 

A thin SVD supplies the reference
solution and the leverage scores, and is reused when varying $\lambda$.
Each sampled system is solved by a Cholesky factorization. For a reused
sketch, this factorization is cached across iterations. Both adaptive and mixture methods draw one initial ridge leverage pilot, cache its Cholesky factor, and evaluate subsequent pilot estimates by triangular solves followed by a matrix vector product. Explicit matrix inverses are not formed.

Let $\widehat{\xb}^{\star}_j=\sum_{\ell=1}^{j}\widetilde{\xb}^{(\ell)}$
denote the accumulated solution after $j$ updates, with
$\widehat{\xb}^{\star}_0=\zero_d$ and $\widehat{\xb}^{\star}_t=\xhatstar$.
We report relative solution error and normalized objective excess,
\[
 \frac{\|\widehat{\xb}^{\star}_j-\xstar\|_2}{\|\xstar\|_2},
 \qquad
 \frac{f(\widehat{\xb}^{\star}_j)-f(\xstar)}{f(\xstar)}.
\]
The objective $f$ is defined in~\eqref{eq:ridge}. To avoid cancellation near
convergence, for $\eb=\widehat{\xb}^{\star}_j-\xstar$ the implementation
evaluates its numerator as
$\|\Ab\eb\|_2^2+\lambda\|\eb\|_2^2$. This equals the objective
excess at the exact minimizer.

The plotted parameter $s$ is the number of columns sampled in each iteration. For adaptive and mixture methods, the initial pilot has
$s_0=s$ columns and each update has
\[
 s^{(j)}=\left\lfloor\frac{st-s_0}{t}\right\rfloor
        =\left\lfloor\frac{s(t-1)}{t}\right\rfloor,
 \qquad j=1,\ldots,t.
\]
Thus the pilot and updates together use at most $st$ sampled columns.

Curves show medians and interquartile ranges over independent sketch
trials. Iteration curves use thin lines without markers; parameter sweeps use
small filled markers at the evaluated settings. The same six-panel layout is used for all three datasets.

\subsection{ARCENE and synthetic data}
\label{app:experiment-baselines}
\label{app:baseline-settings}
These experiments directly compare with the column sampling framework of
\citet{chowdhury2018iterative}. We reproduce their types of iteration,
sketch size, and regularization sweeps, adding the oracle and adaptive
probabilities. Panels (a--d) use fresh sketches for all six distributions;
panel (e) shows the oracle factor along each of the six methods' residual
trajectories, and panel (f) shows the proposal factor for adaptive sampling
and its mixture. These trajectories are the same as in panels (a,b).
A separate experiment compares the original reused-sketch
baselines with their refreshed versions.

\begin{table}[htbp]
\caption{ARCENE and synthetic configurations. Iteration curves contain the
initial iterate and twenty updates. The parameter $s$ follows the budget
convention in Appendix~\ref{app:experiment-implementation}.\\}
\label{tab:experiment-config}
\centering
\begin{tabular}{lll}
\toprule
Setting & ARCENE & Synthetic\\
\midrule
Design dimensions & $200\times10000$ & $500\times50000$\\
Independent sketch trials & $15$ & $5$\\
Iteration panels: $s$ & $2400$ & $20000$\\
Iteration panels: $\lambda$ & $10$ & $10$\\
Sketch size sweep & $500,750,\ldots,2500$ & $2500,5000,\ldots,20000$\\
Regularization sweep: $\lambda$ & $1,2,5,10,20,50$ & $10,20,50,75,100,150$\\
Updates in parameter sweeps & $10$ & $10$\\
Regularization panels: $s$ & $2400$ & $20000$\\
Master seed & $0$ & $42$\\
\bottomrule
\end{tabular}
\end{table}

\paragraph{ARCENE.}
The ARCENE feature selection data \citep{guyon2004result} combine the
provided training and validation arrays into one $200\times10000$ design,
with responses in $\{-1,1\}$. We apply a single global min--max scaling to
this combined design. This experiment studies numerical solution of that
fixed ridge problem; it does not use the combined data to report held-out
classification performance.
Identically zero columns retain zero leverage mass and zero solution
coefficients; uniform sampling still uses all $10000$ columns.

\paragraph{Synthetic construction.}
Following the noisy low rank construction of \citet{chen2015fast}, the
implementation uses $n=500$, $d=50000$, and rank parameter $r=50$. Draw
independent standard Gaussian matrices
$\Mb\in\R^{n\times r}$, $\Gb\in\R^{d\times r}$,
$\Eb\in\R^{n\times d}$, and independent standard Gaussian vectors
$\xb_{\mathrm{gen}}\in\R^d$ and $\xib\in\R^n$.
Let $\Qb\in\R^{d\times r}$ be the orthonormal factor from a thin QR
factorization of $\Gb$, and let $\Db\in\R^{r\times r}$ be
diagonal with $\Db_{kk}=1-(k-1)/d$ for $k=1,\ldots,r$. The fixed design and
response are
\[
 \Ab=\Mb\Db\Qb^{\ts}+0.05\Eb,
 \qquad
 \bb=\Ab\xb_{\mathrm{gen}}+5\xib.
\]
A NumPy random generator with seed $42$ generates these arrays separately
from the sketching streams. The same design and response are used for all
methods and all regularization values.

\begin{figure}[t]
 \centering
 \includegraphics[width=\linewidth]{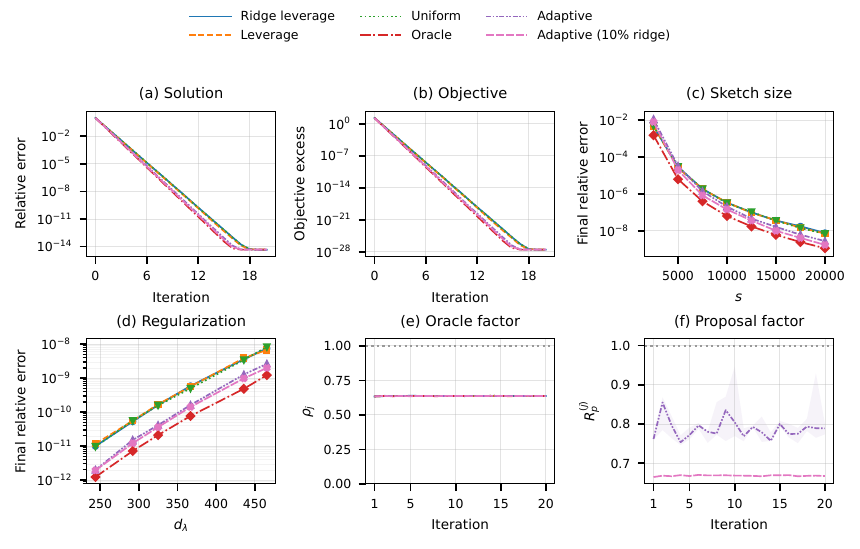}
 \caption{Synthetic comparisons with fresh sketches: solution error and
 normalized objective excess against iteration, final error
 against the column budget and effective degrees of freedom, and oracle
 and proposal factors along the same trajectories. Curves and
 bands are medians and interquartile ranges over 5 trials. Settings are
 given in Table~\ref{tab:experiment-config}. Panel (e) shows all six methods;
 (f) shows adaptive sampling and its mixture. }
 \label{fig:synth}
\end{figure}

\paragraph{Fixed and fresh sketches.}
The separate ARCENE comparison uses $\lambda=10$, 25 updates, and
$s\in\{200,400,800,1600\}$, with 15 trials and master seed $0$.
The reused sketch row contains uniform, leverage, and ridge leverage
sampling. The fresh sketch row contains these three laws and the oracle,
adaptive, and mixture proposals. These rows isolate the effect of refreshing
each baseline distribution, while also showing the added probabilities.

\begin{figure}[t]
 \centering
 \includegraphics[width=\linewidth]{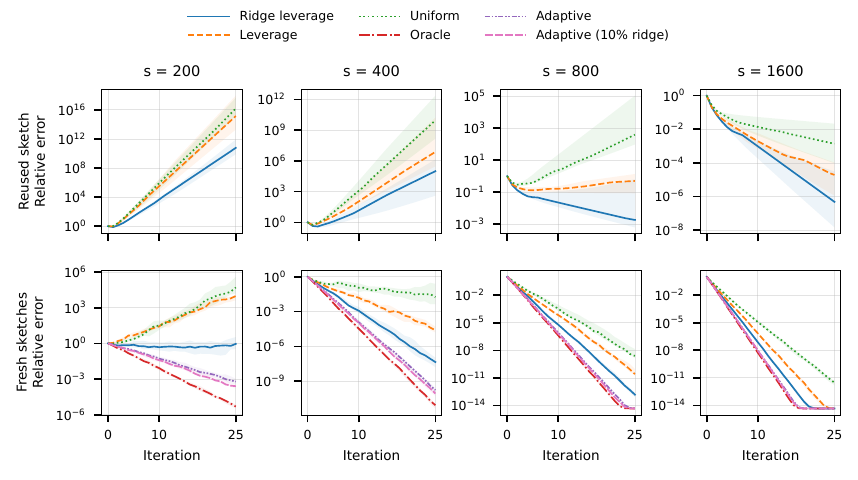}
 \caption{Reused (top) and fresh (bottom) sketches on ARCENE. Columns use
 $s=200,400,800,1600$; all panels use $\lambda=10$ and 25 updates.
 Lines and bands are
 medians and interquartile ranges over 15 trials. Both rows include
 the three baseline laws; the fresh row also includes our oracle,
 adaptive and mixture proposals. Panels use separate vertical scales.}
 \label{fig:fresh}
\end{figure}

\subsection{Ridge probes on Qwen2.5 representations}
\label{app:experiment-qwen}
\label{app:qwen-settings}
We use frozen Qwen2.5-32B-Instruct representations
\citep{qwen2025technical} of $N=100$ STS Benchmark training examples
\citep{cer2017stsb}, with similarity scores from 0 to 5 as labels.
We concatenate representations from 20 layers at the final retained
non-padding token. Each layer contributes $5120$ coordinates, giving
$d=102{,}400$ features per example.
The layer indices, in concatenation order, are
\[
64,32,16,48,8,24,40,56,4,12,20,28,36,44,52,60,6,22,38,54.
\]

\paragraph{Centering and the fitted intercept.}
Let $\Fb\in\R^{N\times d}$ denote the raw training feature matrix,
$\yb\in\R^N$ its labels, and $\one\in\R^N$ the vector of ones.
We convert the stored features and labels from FP32 to FP64 before
computing the training means
$\mub=\Fb^{\ts}\one/N$ and $\bar y=\one^{\ts}\yb/N$.
The centered design and labels are
\[
\Ab=\Fb-\one\mub^{\ts},\qquad
\bb=\yb-\bar y\one.
\]
This accounts for an unpenalized intercept
$\bar y-\mub^{\ts}\xb$.
The centered design has numerical rank $99$; positive regularization
keeps both the exact and sketched ridge systems positive definite.

\paragraph{Regularization selection.}
We select $\lambda$ by five-fold cross-validation on the 100 training
examples, using shuffled folds with seed 2026 and the candidate grid
$\{10^{-4+k/2}:k=0,\ldots,24\}$.
The penalty $\lambda\|\xb\|_2^2$ is added to the sum of squared errors.
Within each fold, features and labels are centered using training-fold
means only.
We choose the candidate minimizing the mean fold MSE, giving
$\lambda=\QwenLambda$, and refit on all 100 training examples.

\paragraph{Sampler comparisons and settings.}
Figure~\ref{fig:qwen} compares all six sampling methods on this one design,
using 30 trials and master seed $20260922$. Panels (a,b,e,f) use
$s=1000$ and 20 updates at the selected regularization value. The
sketch size sweep uses $s\in\{500,1000,2000,4000,8000\}$ and 10 updates;
the regularization sweep uses $\lambda/4,\lambda/2,\lambda,2\lambda,4\lambda$
with $s=1000$ and ten updates. It measures numerical sensitivity around the
selected value without selecting a new regularization parameter.

\paragraph{Data and feature construction.}
We use the official training split of the Semantic Textual Similarity
Benchmark (STS-B)~\citet{cer2017stsb}, accessed through
Hugging Face (\texttt{nyu-mll/glue}, configuration \texttt{stsb}).
We shuffle this split with seed 2026 and retain the first 100 examples.
Each sentence pair is inserted into the following user message:
\begin{quote}
\small\ttfamily
Consider the following two sentences.

Sentence 1: [sentence1]\\
Sentence 2: [sentence2]

Assess their semantic similarity.
\end{quote}
We apply the model's chat template with an assistant generation prefix,
then tokenize without adding further special tokens.
Inputs are right-truncated to 128 tokens and right-padded within batches.
Feature extraction uses 4-bit NF4 weights with FP16 computation;
no response is generated.
No feature variance standardization is applied.
The code repository includes metadata files recording the selected
example IDs and the dataset, model, and tokenizer revision identifiers.

\subsection{Variance factor measurements}
\label{app:experiment-factors}
At iteration $j=1,\ldots,t$, after updating $\bb^{(j)}$ and before drawing
$\Sb^{(j)}$, we evaluate the exact solution
$\xb^{\star(j)}$ to the current residual problem using the
cached SVD. With the ridge leverage weights $w_i$ and effective dimension
$\dlam$ from Section~\ref{sec:setup}, we record
\[
 \rho_j=\frac{\big(\sum_{i=1}^d|\xb_i^{\star(j)}|\sqrt{w_i}\big)^2}
                  {\dlam\|\xb^{\star(j)}\|_2^2},
 \qquad
 R_p^{(j)}=\frac{\sum_{i=1}^d |\xb_i^{\star(j)}|^2w_i/p_i^{(j)}}
                    {\dlam\|\xb^{\star(j)}\|_2^2},
\]
where $p^{(j)}$ is the distribution actually used for that update. The
first quantity is the ideal oracle factor for the current residual; the
second measures the actual proposal relative to ridge proposal. In particular, $R_p^{(j)}=\rho_j$ for
the oracle and $R_p^{(j)}=1$ for ridge leverage sampling. For adaptive
trajectories we measure both using the exact residual solution, rather than
substituting the pilot estimate in the diagnostic. The exact solution is
used only for the oracle sampler and for evaluation; the practical adaptive and mixture
samplers use their pilot estimates. Different methods generally encounter
different residuals after the first update.

\subsection{Gaussian reference for the oracle factor}
\label{app:rho-gaussian}

\begin{proof}[Proof of Proposition~\ref{prop:rho-gaussian}]
Taking expectations over the sampled index gives
\[
\mathbb E_i\zeta^2=1,
\qquad
\mathbb E_i|\zeta|
=
\frac{\sum_i|\xb_i|\sqrt{w_i}}
{\sqrt{\dlam}\,\twonorm{\xb}}
=
\sqrt{\rho(\xb)}.
\]
Since $z\mapsto|z|$ is $1$-Lipschitz and the standard normal
absolute first moment is $\sqrt{2/\pi}$,
\[
\left|\sqrt{\rho(\xb)}-\sqrt{\frac{2}{\pi}}\right|
\le\Delta.
\]
Consequently,
\[
\begin{aligned}
\left|\rho(\xb)-\frac{2}{\pi}\right|
&=
\left|\sqrt{\rho(\xb)}-\sqrt{\frac{2}{\pi}}\right|
\left(\sqrt{\rho(\xb)}+\sqrt{\frac{2}{\pi}}\right)\\
&\le
\Delta\left(2\sqrt{\frac{2}{\pi}}+\Delta\right).
\end{aligned}
\]
\end{proof}

\paragraph{Empirical assessment.}
We compute $\Delta$ for the Qwen2.5 ridge probes to assess departures
of $P_\zeta$ from $\mathcal N(0,1)$.
Figure~\ref{fig:rho-gaussian-w1} shows small distances and an overall
downward trend with increasing model size and feature dimension.

\begin{figure}[t]
\centering
\includegraphics[width=\linewidth]{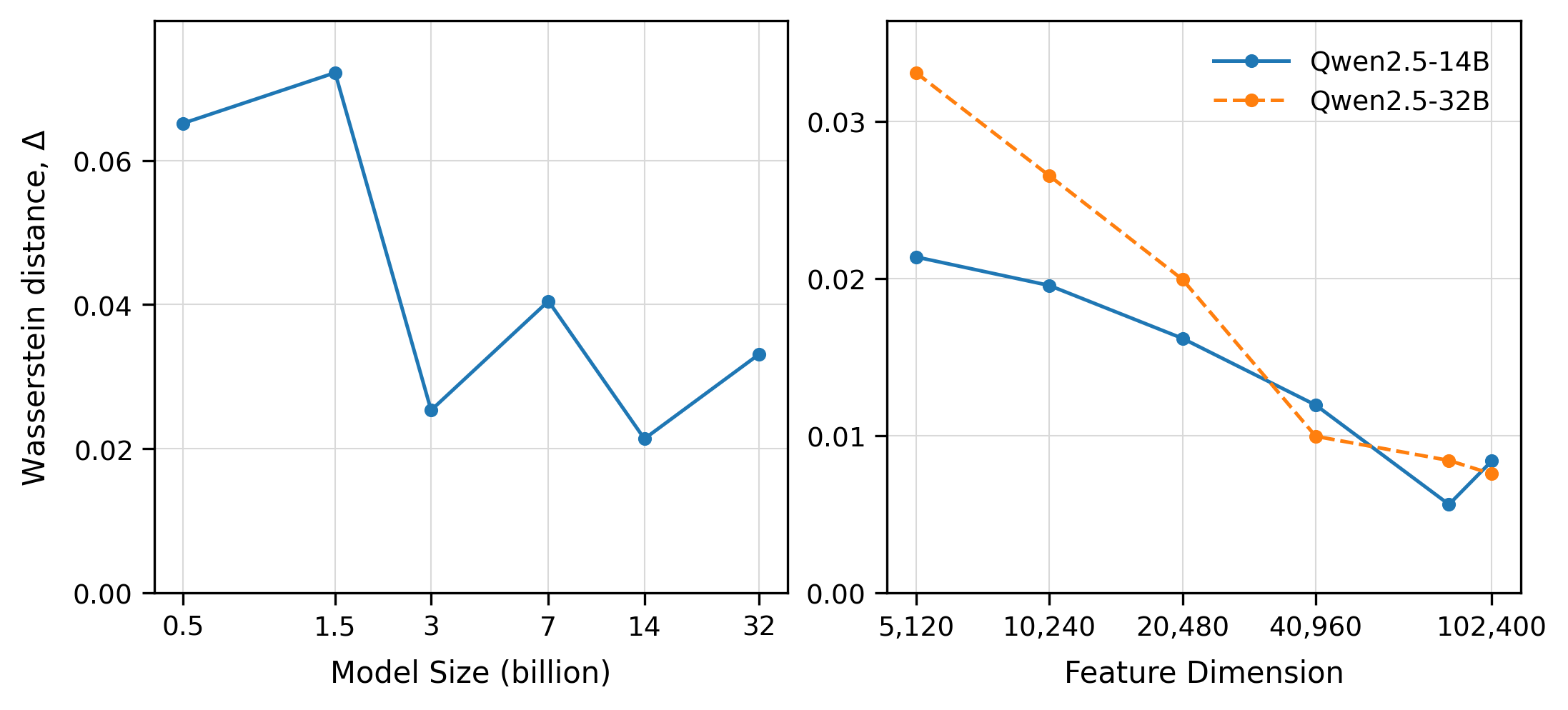}
\caption{Wasserstein distance
$\Delta=W_1(P_\zeta,\mathcal N(0,1))$
for Qwen2.5 ridge probes.
Left: final-layer representations across model sizes.
Right: representations concatenated from 1, 2, 4, 8, 16, or 20 layers
of Qwen2.5-14B and Qwen2.5-32B.
Both horizontal axes are logarithmic.}
\label{fig:rho-gaussian-w1}
\end{figure}
\end{document}